\documentclass[letterpaper]{article} % DO NOT CHANGE THIS
\usepackage{aaai24}  % DO NOT CHANGE THIS
\usepackage{times}  % DO NOT CHANGE THIS
\usepackage{helvet}  % DO NOT CHANGE THIS
\usepackage{courier}  % DO NOT CHANGE THIS
\usepackage[hyphens]{url}  % DO NOT CHANGE THIS
\usepackage{graphicx} % DO NOT CHANGE THIS
\usepackage{natbib}  % DO NOT CHANGE THIS AND DO NOT ADD ANY OPTIONS TO IT
\usepackage{caption} % DO NOT CHANGE THIS AND DO NOT ADD ANY OPTIONS TO IT
\usepackage{algorithm}
\usepackage{algorithmic}

\usepackage[utf8]{inputenc} % allow utf-8 input
\usepackage{url}            % simple URL typesetting
\usepackage{booktabs}       % professional-quality tables
\usepackage{amsfonts}       % blackboard math symbols
\usepackage{nicefrac}       % compact symbols for 1/2, etc.
\usepackage{microtype}      % microtypography
\usepackage{xcolor}         % colors
\usepackage{soul}
\usepackage{graphicx}
\usepackage{bm}
\usepackage{amsmath} 
\usepackage{amssymb} 
\usepackage{dsfont}
\usepackage{multirow} 
\usepackage{amsthm}
\usepackage{mathrsfs}
\usepackage{amsmath}
\usepackage{amsfonts}
\usepackage{amssymb}
\usepackage{mathtools}
\usepackage{xspace}
\usepackage{thmtools} 
\usepackage{thm-restate}

\newcommand{\R}{{\mathbb{R}}}

\newcommand{\Em}{{\mathbb{E}}}
\newcommand{\Pm}{{\mathbb{P}}}

\newcommand{\bx}{{\bm{x}}}
\newcommand{\bp}{{\bm{p}}}
\newcommand{\bz}{{\bm{z}}}

\newcommand{\bg}{{\bm{g}}}
\newcommand{\by}{{\bm{y}}}

\newcommand{\be}{{\bm{e}}}
\newcommand{\T}{{\mathsf{T}}}
\newcommand{\ber}{{\mathsf{Ber}}}

\newcommand{\gvn}{\,|\,}

\newcommand{\I}{\mathds{1}}

\newcommand{\cS}{{\cal{S}}}
\newcommand{\cD}{{\cal{D}}}
\newcommand{\sM}{{\mathscr{M}}}
\newcommand{\WS}{{W_{\cal{S}}}}
\newcommand{\simplex}{\Delta}
\newcommand{\argmax}{\mathrm{argmax}}
\newcommand{\argmin}{\mathrm{argmin}}

\newcommand{\RR}{{\rm RR}}

\usepackage[nameinlink]{cleveref}       
\crefname{equation}{Eq.}{Eqs.}

\newtheorem{defn}{Definition}
\usepackage{marginnote}
\usepackage{relsize}
\title{Robust Loss Functions for Training Decision Trees with Noisy Labels}
\author {
    Jonathan Wilton,
    Nan Ye
}
\affiliations {
    School of Mathematics and Physics, 
    The University of Queensland\\
    jonathan.wilton@uq.edu.au, 
    nan.ye@uq.edu.au
}
\begin{document}

\maketitle

\begin{abstract}
	We consider training decision trees using noisily labeled data, focusing on
	loss functions that can lead to robust learning algorithms.
	Our contributions are threefold.
	First, we offer novel theoretical insights on the robustness of many existing
	loss functions in the context of decision tree learning.
	We show that some of the losses belong to a class of what we call
	\emph{conservative losses}, 
    and the conservative losses % which?
    lead to
	an early stopping behavior during training and noise-tolerant predictions
	during testing.
	Second, we introduce a framework for constructing robust loss functions,
	called \emph{distribution losses}.
	These losses apply percentile-based penalties based on an assumed margin
	distribution, and they naturally allow adapting to different noise rates via a
	robustness parameter.
	In particular, we introduce a new loss called the \emph{negative exponential loss},
	which leads to an efficient greedy impurity-reduction learning
	algorithm.
	%which leads to an impurity that can nicely interpolate between the nonrobust
	%Gini impurity and the robust misclassification impurity.
	Lastly, our experiments on multiple datasets and noise settings validate our
	theoretical insight and the effectiveness of our adaptive negative exponential
	loss.
\end{abstract}
% additional things to possibly mention
% - polish up the sentence on the theoeretical analysis (make it readable for a general audience, and a more accurate summary of the results)
% - provide intuition on the CDF framework and the NE impurity
% - adaptiveness of the new impurity measure

\section*{Introduction}

Noisily labeled data often arise in machine learning, due to reasons such as 
the difficulty of accurately labeling data and
the use of crowd-sourcing for labeling \cite{song2022learning}.
Various approaches have been developed to handle the label noise, including 
eliminating mislabeled examples \cite{brodley1996identifying},
implicit/explicit regularization \cite{tanno2019learning,lukasik2020does},
the use of robust loss functions \cite{manwani2013noise,yang2019robust},
with recent works mostly focusing on neural networks.

This paper focuses on robust loss functions for learning decision trees from
noisily labeled data.
Tree-based methods (e.g., random forests) are among the most effective machine
learning methods, particularly on tabular data
\cite{grinsztajn2022tree,kaggle2021state},
and several robust loss functions have been shown to be effective for neural
network learning in the presence of label noise (e.g., see
\cite{ghosh2017robust,zhang2018generalized}).
However, little attention has been paid to the understanding and design of
robust loss functions in the context of decision tree learning.
This is likely because decision tree learning algorithms are often described as
greedy impurity-reduction algorithms in the literature, and it is less
well-known that the impurity-reduction algorithms are greedy algorithms for
minimizing certain losses \cite{yang2019robust,wilton2022positive}.
Our work aims to address this research gap.

Our main contributions are three-fold.
\begin{itemize}
	\item We offer novel theoretical insight on the robustness of many existing
		loss functions.
		We show that some of them belong to a class of what we call
		\emph{conservative} losses, which are robust due to an early stopping
		behavior during training and noise-tolerant predictions during testing.
	\item We introduce a framework for constructing robust loss functions, called
		\emph{distribution losses}.
		These losses apply percentile-based penalties based on an assumed margin
		distribution.
		By using different assumed margin distributions, we can recover some commonly
		used loss functions, which shed interesting insight on these existing
		functions.
		An attractive property of the distribution loss is that they naturally allow
		adapting to different noise rates via a robustness parameter. 
		Importantly, we introduce a new loss called the \emph{negative exponential loss},
		which leads to an efficient impurity-reduction learning algorithm.
		%which leads to an impurity that can nicely interpolate between the nonrobust
		%Gini impurity and the robust misclassification impurity.
		%As far as we know, previous work on decision tree specific robust loss only
		%considered binary classification, while our 
	\item Our extensive experiments validate our theoretical insight and the
		effectiveness of our adaptive negative exponential loss. 
\end{itemize}

% We make the following contributions in this paper.
% \begin{itemize}
%     \item theoretical results supporting our claims of robustness, including in finite data setting
%     \item a framework for constructing robust loss functions based on CDFs
%     \item novel class of impurity measures for robust trees using this framework
%     \item extensive experiments including multiclass classification and different label noise settings
% \end{itemize}

% noteworthy results:
% \begin{itemize}
%     \item misclassification impurity robust to high levels of label noise
%     \item robustness of new impurity (which is simple drop-in replacement for existing impurity measures) controlled with hyperparameter (adaptive to underlying noise levels), can be tuned effectively using noisy validation data
%     \item performance of new impurity roughly $\geq$ max(robust, non-robust) impurities
% \end{itemize}

The remainder of this paper is organized as follows.
We first provide a more detailed discussion on related work, 
followed by some preliminary concepts.
We then present the conservative losses and their robustness
properties.
After that, we describe our distribution-based robust loss framework
and the negative exponential loss.
Finally, we present details on experimental settings and results, with a brief conclusion.
Our source code is available at \url{https://github.com/jonathanwilton/RobustDecisionTrees}.
% \Cref{sec:related} provides a more detailed discussion on related work and 
% \Cref{sec:background} presents some preliminary concepts.
% \Cref{sec:theory} presents the conservative losses and their robustness
% properties.
% %In \Cref{sec:new-framework} we introduce a new framework for creating robust loss functions for classification, with either noisy or noise-free labels, based off cumulative distribution functions from probability theory. 
% %We then use this framework to discover a class of loss functions, which we collectively call the negative-exponential (NE) loss functions, and associated impurity measures, which uses a hyperparameter to effectively trade-off between robust and non-robust impurities and control the degree to which early stopping is applied during training.
% \Cref{sec:new-framework} describes our distribution-based robust loss framework
% and the negative exponential loss.
% \Cref{sec:experiments} presents details on experimental settings and results.
% Our source code is available at \url{https://github.com/jonathanwilton/RobustDecisionTrees}.

\section*{Related Work} \label{sec:related}

% learning with label noise 
Our work is broadly related to the large body of approaches developed for
dealing with label noise in the literature, which include filtering the noisy
labels, learning a classifier and model for the label noise simultaneously,
implicit/explicit regularization to avoid overfitting the noise,
and designing robust loss functions (see, e.g., for the excellent reviews
\cite{frenay2013classification} or \cite{song2022learning} for detailed
discussions).

% robust loss functions
The most relevant general approach to our work is the robust loss approach.
There are two common approaches to design robust losses.
One creates corrected losses by incorporating label noise rates if they are known 
(e.g., see \cite{natarajan2013learning,patrini2017making}).
%\cite{natarajan2013learning,sukhbaatar2014training,van2015machine,menon2015learning,patrini2017making}.
However, such information is typically unknown, and difficult to estimate accurately in practice.
Another approach considers inherently robust losses, which does not require
knowledge of the noise rates.
Some pioneering theoretical works consider robustness of losses against label noise 
\cite{manwani2013noise, ghosh2015making, ghosh2017robust}.
These works show that the zero-one (01) loss and mean absolute error (MAE) are
robust to many types of label noise, while the commonly used cross entropy (CE)
loss does not enjoy these same robustness properties.
This sparked interest in developing new loss functions that share favorable qualities from each of the 01, MAE and CE, for example the 
generalized cross entropy (GCE) loss \cite{zhang2018generalized}, 
negative learning \cite{kim2019nlnl}, 
symmetric cross entropy loss \cite{wang2019symmetric}, 
curriculum loss \cite{Lyu2020Curriculum} and 
normalized loss functions \cite{ma2020normalized}.
However, losses like the curriculum loss and the normalized losses are not
suitable for decision tree learning, because the impurities for these losses lack
analytical forms and efficient algorithms.

% notable is the generalized cross entropy (GCE) loss that has special cases of MAE and CE, and can be used as a simple plug in replacement for training neural nets. We want to bring something similar to DTs.

% loss functions that require label corruption rates
%Another direction for creating new losses is to incorporate label corruption rates if they are known 
%\cite{natarajan2013learning,sukhbaatar2014training,van2015machine,menon2015learning,patrini2017making}.
%However, such information is typically unknown, and difficult to estimate accurately in practice.

% decision tree learning with label noise 
Our work is also closely related to tree methods for learning from noisily
labeled data.
Motivated by the predictive performance and interpretability of tree methods
\cite{breiman1984classification,breiman2001random,geurts2006extremely}, 
various works have developed algorithms for decision tree learning in the
presence of label noise, such as pruning \cite{breiman1984classification},
making use of pseudo-examples during tree construction \cite{mantas2014credal},
leaving a large number of samples at each leaf node \cite{ghosh2017robustness}, 
and adjusting the labels at each leaf node of a trained RF \cite{zhou2019improving}.
To the best of our knowledge, the closest work to ours derives a robust impurity
from the well-known ranking loss \cite{yang2019robust}.
They only consider binary classification and their approach does not adapt to
the noise rate, while we also consider multiclass classification and our
approach adapts to the underlying noise rate.
In addition, we offer novel theoretical insight on the robustness of various
existing losses, and we contribute a general framework for constructing robust
losses.

%tree construction procedures derived from loss functions that are robust to label noise. 
%Our approach is adaptive to noise level and, in turn, gives strong performance for many label noise settings without requiring knowledge of the label corruption process. 

%While \citet{yang2019robust} considered only binary classification and a loss
%that does not adapt to the noise rate, 
%robustness of existing loss functions and the design of general robust losses that can adapt to the noise
%rate.

% In spite of tree-based methods having strong performance in many tasks,
% little attention has been given to them in the learning with noisy labels literature. Popular tree growing methods are implicitly performing recursive greedy risk minimization. In light of this, in this paper we investigate the use of robust loss functions in tree growing procedures.

\section*{Preliminaries}\label{sec:background}

\paragraph{Learning With Noisy Labels}

We consider $K$-class classification problem, where the input
$\bx\in\R^d$ and the one-hot label $\by\in\{0,1\}^K$ follows a joint
distribution $p(\bx,\by)$.
In the standard noise-free setting, we are given a training set
$\cD=\{(\bx_i,\by_i)\}_{i=1}^n$ consisting of examples independently sampled
from $p(\bx, \by)$.
In the noisy setting, we have a dataset
$\widetilde{\cD}=\{(\bx_i,\widetilde{\by}_i)\}_{i=1}^n$ 
where each noisy label $\widetilde{\by}$ is obtained by randomly flipping the
true label $\by$ with probability 
$\eta_{jk}^{\bx} := \Pm(\widetilde{\by} = \be_k\gvn \by = \be_j,\bx)$,
with $\be_{j}$ being the one-hot vector for class $j$.
We focus on class-conditional noise, where 
the noise probability is independent of the input, that is, each
$\eta_{jk}^{\bx}$ is equal to some constant $\eta_{jk}$ for all $\bx$.
In particular, we consider the special case of uniform noise, in which each
class has the same corruption rates, that is, $\eta_{jk}=1-\eta$ for $j=k$ and
$\eta_{jk}=\eta/(K-1)$ for $j\ne k$, for some constant $\eta$.

The objective is to learn a classifier $\bg:\R^d\to \simplex$ to minimize the expected risk
\[
    R(\bg) := \Em_{(\bx,\by)\sim p(\bx,\by)}\, \ell(\bg(\bx),\by),
\]
where $\ell:\R^{K}\times \R^K\to \R$ is a loss function, and  
$\simplex = \{\by\in [0,1]^K\,:\, \by^\T\bm{1}=1\}$ the standard
$(K-1)$-simplex.
Predicted labels can be obtained from the classifier with $\be_{\argmax(\bg(\bx))}$. 
With a set of noise-free training data, the risk can be estimated without bias via the empirical risk 
\[
    \widehat{R}(\bg;\cD) := \sum_{(\bx,\by)\in\cD} \ell(\bg(\bx),\by) / |\cD|.
\]
When only a noisily labeled dataset $\widetilde{\cD}$ is available, we instead
estimate the expected risk with $\widehat{R}(\bg;\widetilde{\cD})$.

We focus on loss functions that are robust against label noise.
A loss function $\ell$ is said to be noise tolerant if the minimizers of the
expected risk using loss $\ell$ on the noisy and noise-free data distributions
lead to the same expected risk using the 01 loss on noise-free data
\cite{manwani2013noise}.
If a loss function $\ell$ is symmetric, i.e., 
$\sum_{j=1}^K \ell(\bg(\bx),\be_j)=C$ for any $\bx\in\R^d$ and any $\bg$,
% \begin{equation}
%     \sum_{j=1}^K \ell(\bg(\bx),\be_j)=C,\quad \forall\bx\in\R^d,\,\forall \bg, \label{eq:symm}
% \end{equation}
then, under uniform label noise with $\eta<\frac{K-1}{K}$, $\ell$ is noise tolerant \cite{ghosh2015making,ghosh2017robust}. 
If we additionally have 
$R(\bg^*)=0$ for some classifier $\bg^{*}$, 
$0\leq \ell(\bg(\bx),\be_j)\leq C/(K-1)\,\forall j=1,\ldots,K$ 
and the matrix 
$(\eta_{ij})_{i,j=1}^K$ 
is diagonally dominant, then $\ell$ is noise tolerant under class conditional noise \cite{ghosh2017robust}.
Examples of symmetric loss functions include the $01$ loss $\ell(\widehat{\by},\be_j)=\I(\widehat{\by}\ne \be_j)$ 
and MAE loss $\ell(\widehat{\by},\be_j) = \|\widehat{\by}-\be_j\|_1$. 
On the other hand, the mean squared error (MSE) loss
$\ell(\widehat{\by},\be_j)=\|\widehat{\by}-\be_j\|_2^2$ and CE
loss $\ell(\widehat{\by},\be_j) = -\be_j^\T\log \widehat{\by}$ are not
symmetric.
In practice it has been shown that training neural network (NN) classifiers with these noise tolerant loss functions can lead to significantly longer training time before convergence \cite{zhang2018generalized}. 
The GCE loss $\ell(\widehat{\by},\be_j)=(1-(\be_j^\T
\widehat{\by})^q)/q$ was proposed as a compromise between noise-robustness and
good performance \cite{zhang2018generalized}. The hyperparameter $q\in[0,1]$
controls the robustness, with special cases $q=0$ giving the CE and $q=1$ the
MAE. 

\paragraph{Decision Tree Learning} 
We briefly review two dual perspectives for decision tree learning:
learning by impurity reduction, and learning by recursive greedy risk minimization.

In the impurity reduction perspective, the decision tree construction process
recursively partitions the (possibly noisy) training set $\cD$ such that each
subset has similar labels.
We start with a single node associated with the entire training set.
Each time we have a node associated with a subset $\cS \subseteq \cD$ that we
need to split, we find an optimal split $(f,t)$ that partitions $\cS$ into 
$\cS_{f \leq t}$ and $\cS_{f > t}$ based on whether the feature $f$ is larger
than the value $t$.
The quality of a split is usually measured by its reduction in some label
impurity measure $I$, defined by:
\[
    \frac{|\cS|}{|\cD|}I(\cS) - \frac{|\cS_{f\leq t}|}{|\cD|}I(\cS_{f\leq t}) - \frac{|\cS_{f> t}|}{|\cD|}I(\cS_{f> t}).
\]
Based on the split, two child nodes associated with $\cS_{f \leq t}$ and 
$\cS_{f > t}$ are created. 
This process is repeated recursively on the two child nodes until some stopping
criteria is satisfied. 

% Generalisation to multiclass classification of framework laid out in PUET 
In the recursive greedy risk minimization perspective, a node is split in a greedy way to minimize 
the empirical risk $\widehat{R}(\bg;\cD)$. 
If $\bg$ predicts a constant $\widehat{\by}$ on a subset $\cS\subseteq \cD$ of the training examples, then the contribution to the empirical risk is the partial empirical risk 
$
% \begin{align*}
    \widehat{R}(\widehat{\by};\cS):=\sum_{(\bx,\by)\in\cS}  \ell(\widehat{\by},\by)/|\cD|.
% \end{align*}
$
Denote by 
$\widehat{\by}^*:=\argmin_{\widehat{\by}\in\simplex}\widehat{R}(\widehat{\by};\cS)$ the optimal constant probability vector prediction, with minimum partial empirical risk 
% $
\begin{align*}
    \widehat{R}^*(\cS):=\widehat{R}(\widehat{\by}^*;\cS).
\end{align*}
% $
The minimum partial empirical risk can be interpreted as an impurity measure.
In fact, it is the Gini impurity and the entropy impurity (up
to a multiplicative constant) when the loss is the MSE loss and the CE loss,
respectively.

If we switch from a constant prediction rule to a decision stump that splits on feature $f$ at threshold $t$, then the minimum partial empirical risk for the decision stump is $\widehat{R}^*(\cS_{f\leq t})+\widehat{R}^*(\cS_{f> t})$, and the risk reduction for the split $(f,t)$ is 
\begin{equation}
    \RR(f,t;\cS) := \widehat{R}^*(\cS) - \widehat{R}^*(\cS_{f\leq t}) - \widehat{R}^*(\cS_{f> t}).
\end{equation}
% Some works have identified the connection between loss functions and impurity measures for decision trees, for example, binary classification \cite{wilton2022positive, painsky2018universality}, multiclass classification with MSE/Gini and (kind of) 01/misclassification \cite{breiman1984classification}. 
% However, these results are for binary classification problems and, \textcolor{red}{to the best of our knowledge}, no such links have not been officially established for the multiclass classification setting. 

We will use a subscript to specify the loss if needed. 
For example, $\widehat{R}_{\mathrm{MSE}}^*(\cS)$ and
$\RR_{\mathrm{MSE}}(f,t;\cS)$ indicates that the MSE loss is used for computing
the minimum partial empirical risk and the risk reduction for a split $(f, t)$
on $\cS$, respectively.

The equivalence between loss functions and impurity measures has been explored
in, for example, 
%\cite{breiman1984classification, yang2019robust, wilton2022positive}.
\cite{yang2019robust, wilton2022positive}.
To illustrate, let $\bp$ be the empirical class distribution for $\cS$. 
Then the minimum partial risks for MSE and CE are the commonly used Gini
impurity $1-\|\bp\|_2^2$ and entropy impurity $-\bp^\T\log \bp$, respectively, up
to a multiplicative constant (see (a) and (b) of \Cref{thm:loss_impurity_equivalence}). 
Note that not all loss functions have impurities which have analytical forms and
efficient algorithms, while our negative exponential loss yields an 
impurity which has an efficiently computable analytical formula, which is 
important for efficient decision tree learning.
%See, for example, the implementation of decision tree classifier in scikit-learn \cite{pedregosa2011scikit}.

For prediction, each test example is assigned to a leaf node based on its
feature values, then labeled according to the majority label of the examples at
the leaf node. 
Generalization performance of the decision tree classifier can be measured by
comparing predictions on unseen data with true labels, however, it can be
heavily affected when training data, particularly labels, are unreliable. 

\section*{Conservative Losses}
\label{sec:theory}

We first examine the impurities corresponding to various loss functions,
including both standard loss functions and robust ones, in
\Cref{thm:loss_impurity_equivalence}.
Parts $(a)$, $(b)$ and $(c)$ are shown in previous works
\cite{breiman1984classification, painsky2018universality,yang2019robust,
wilton2022positive}
but included for completeness.
All proofs are given in the appendices.
\footnote{The appendices are available in the full version of the paper at
\url{https://github.com/jonathanwilton/RobustDecisionTrees}.}
Unless otherwise stated, we shall use $\cD$ to denote an arbitrary (possibly
noisy) set of input-output pairs, and $\cS$ a subset of $\cD$, in this section.

\begin{restatable}{theorem}{lie}
\label{thm:loss_impurity_equivalence}
    Let 
    $W_{\cS}=|\cS|/|\cD|$
    and 
    $\bp
    =(p_1,\ldots,p_K)^\T
    \in \simplex$ 
    be the empirical class probability vector for $\cS$, that is, 
    $p_j=\sum_{(\bx,\by)\in\cS}\I(\by=\be_j)/|\cS|,\forall j=1,\ldots,K$. 
    Then,

    \begin{enumerate}
        \item[(a)] $\widehat{R}_{\mathrm{MSE}}^*(\cS)=W_{\cS}(1-\|\bp\|_2^2)$,\label{result a}
        \item[(b)] $\widehat{R}_{\mathrm{CE}}^*(\cS)=W_{\cS}(-\bp^\T\log\bp)$,
        \item[(c)] $\widehat{R}_{\mathrm{01}}^*(\cS)=W_{\cS}(1-\|\bp\|_\infty)$,
        \item[(d)] $\widehat{R}_{\mathrm{GCE}}^*(\cS)
            =
            \begin{cases}
                W_{\cS}(-\bp^\T\log\bp), & q=0, \\
                W_{\cS}(1-\|\bp\|_{1/(1-q)})/q, & \forall q\in(0,1), \\
                W_{\cS}(1-\|\bp\|_\infty)/q, & q\geq 1,                 
            \end{cases}$
        \item[(e)] $\widehat{R}_{\mathrm{MAE}}^*(\cS)=2W_{\cS}(1-\|\bp\|_\infty)$.
    \end{enumerate}
\end{restatable}

The result highlights an interesting observation on the impurities of 
two robust losses, the MAE loss and the 01 loss \cite{ghosh2017robust}:
both lead to the misclassification impurity 
$1 - \|\bp\|_{\infty}$, up to a multiplicative constant.
Furthermore, the GCE is equivalent to an impurity measure that interpolates
between the misclassification and entropy impurities depending on the chosen
value of $q$. 
This is consistent with the fact that GCE interpolates between CE and MAE loss
for $q \in (0, 1)$ \cite{zhang2018generalized}.

Below, we introduce a broad class of losses that lead to the misclassification
impurity, and provide a few results to justify their robustness properties in
the context of decision tree learning with label noise.
\begin{defn} \label{defn:conservative}
A loss function $\ell$ is called \emph{$C$-conservative} if,
for some constant $C>0$,
it satisfies the following properties:
\begin{enumerate}
	\item[(a)] $\sum_{j=1}^K\ell(\widehat{\by},\be_j) \geq C(K-1),\ 
		\forall\widehat{\by}\in \simplex$,
	\item[(b)] $\ell(\widehat{\by},\be_j)\leq C,\ 
		\forall\widehat{\by}\in \simplex,\ \forall j=1,\ldots,K$, and
	\item[(c)] $\ell(\be_j,\be_j)=0,\ \forall j=1,\ldots,K$.
\end{enumerate}
\end{defn}
Intuitively, (a) and (b) implies that the loss assigned to a single class is
never too much as compared to the total loss assigned to all classes.
%TODO: discuss the relationship between conservative losses and symmetric losses
%class of loss functions that are both symmetric and L(e_j,e_j)=0 is a subset of conservative loss functions.

\begin{restatable}{theorem}{universal}
\label{thm:universal} 
    In a $K$ class classification problem let $\ell:\R^K\times \R^K\to [0,\infty)$ be a loss function, 
    % $\be_k\in\{0,1\}^K$ be a one-hot vector with 1 in position $k$, 
    and
		$\bp\in\simplex$ be the vector of proportions of examples in $\cS$ from each class. 
		Then, we have 
    \begin{align}
        \widehat{R}^*(\cS) &= C\WS (1-\|\bp\|_\infty) \label{eq:misclassification_impurity}
        % \underset{\widehat{\by}\in \Delta^{K-1}}{\min}\ \sum_j p_k\ell(\widehat{\by},\be_j) &= C(1-\|\bp\|_\infty)
        % \\
        % &= C\sum_{k\ne \argmax(\bp)} p_k
        % \\
        % &= \sum_{k\ne \argmax(\bp)} Cp_k + 0p_{\argmax(\bp)}
    \end{align}
    if $\ell$ is $C$-conservative.
    In addition, if \Cref{eq:misclassification_impurity} holds, then $\ell$ satisfies (a)
    in \Cref{defn:conservative}.
\end{restatable}
This theorem provides a convenient way to check if a loss function leads to the
misclassification impurity, as illustrated in \Cref{cor:robust_losses}. 

\begin{restatable}[]{corollary}{rl}
\label{cor:robust_losses} 

    The MAE, 01 loss, GCE $(q\geq 1)$ and infinity norm loss satisfy
    % \[
    %     \underset{\widehat{\by}\in\cP}{\min}\ \frac{1}{n}\sum_{i=1}^n \ell(\widehat{\by},\by_i) = \underset{\widehat{\by}\in\cP}{\min}\ \sum_j^K p_j\ell(\widehat{\by},\be_j) = C(1-\|\bp\|_\infty), 
    % \]
    $
        \widehat{R}^*(\cS) = W_{\cS}\, C(1-\|\bp\|_\infty),
    $
    with $C=2,\,1,\,1/q$ and $1$, respectively. 
% \end{corollary}
\end{restatable}

Our first robustness property of the conservative loss is concerned with the
optimal predictions.
%Firstly, the misclassification impurity only depends on the proportion of samples in the most prevalent class, rather than all classes like for the commonly used Gini and entropy impurities. This suggests that the impurity measure is less likely to differ significantly when data at a node has errors in the labels compared to other methods. 
%Moreover, the optimal constant probability vector prediction at a node is 
%$\widehat{\by}^* = \be_{\argmax(\bp)}$ 
%when using conservative loss functions, whereas we have 
%$\widehat{\by}^* = \bp$
%for loss functions such as MSE and CE.

% \textcolor{red}{\Cref{cor:one-hot-predictions} was originally included to formalise this point, but I think it may be obvious by simply looking at the form of the different impurity measures.}
% \begin{corollary}
\begin{restatable}[]{theorem}{ohp}
\label{cor:one-hot-predictions}
    Assume $\ell$ is a conservative loss function. Then,
    \[\underset{\widehat{\by}\in \simplex}{\argmin}\,\widehat{R}_{}(\widehat{\by};\cS) = \be_{\argmax(\bp)}.\]
    Moreover, for non-conservative loss functions MSE and CE,
    \[
        \underset{\widehat{\by} \in \simplex}{\argmin}\,\widehat{R}(\widehat{\by};\cS) = \bp.
    \]
\end{restatable}
This result suggests that the optimal constant prediction at each node is less
likely to change after label corruption when the loss functions are conservative
versus not.
This is because for a conservative loss, the optimal constant prediction is the
one-hot vector for the most likely class, which is likely to remain the same
after label noise is added.

Our second robustness result, \Cref{thm:hoeffding}, provides a more precise
statement on the effect of noise on the majority class.
Specifically, we show that a sample size of $O(1/\gamma^{2})$ is needed to
guarantee that the majority class remains the same under the label noise, where
$\gamma$ is a margin parameter that depends on the noise and the class
distribution of the clean dataset, as defined in the theorem below.

\begin{restatable}[]{theorem}{hoeffding} \label{thm:hoeffding}
	%Let $\cS = \{(\bx_{1}, \by_{1}), \ldots, (\bx_{n}, \by_{n})\}$ be a set of $n$
	Let $\cS$ be a set of $n$ noise-free examples, $p_{k}$ be the empirical
	probability of class $k$ in $\cS$, and $k^*$ the most prevalent class in
	$\cS$, that is, $p_{k^*}>p_{k}$ for any $k\ne k^*$.
	In addition, let $\widetilde{\cS}$ be obtained from $\cS$ by applying a
	uniform noise with rate $\eta < (K-1)/K$, and $\widetilde{p}_k$ the empirical
	probability of class $k$ in $\widetilde{\cS}$. 
	Then
	\begin{align*}
		\Pm(\widetilde{p}_{k^*}\geq \widetilde{p}_{k} \text{ for all $k \neq k^{*}$}
		\gvn \cS) 
		&\geq 1-(K-1) e^{-n\gamma^2/2},
	\end{align*}
	with 
	$\gamma = \underset{k\ne k^*}{\min}\,(\widetilde{p}_{k^*}^{\eta} - \widetilde{p}_{k}^{\eta})$,
	and 
	$\widetilde{p}_{j}^{\eta}=\Em [\widetilde{p}_j\gvn \cS],\ \forall\,j=1,\ldots,K$.
\end{restatable}

%TODO: connect to robustness analysis in previous papers
%\cite{zhou2019improving,ghosh2017robustness} 

%TODO: a more interesting question: given a random $\cS$ and a randomly
%corrupted $\widetilde{p}$, what's the probability that both share the same most
%likely class?

%TODO: Does a similar result hold for more general assumptions on the noise?

%TODO (add back this sentence?): In our experiments we observe that the misclassification impurity also offers robustness to class conditional noise.

%This result says that, under uniform noise, the lower bound on the probability that the majority class for the noisy and nose-free data remains the same increases with the margin $\gamma$. 
%Using \Cref{thm:hoeffding}, we can assure the lower bound is larger than some probability $\alpha\in[0,1)$ if
%\begin{align*}
    %% 1-(K-1)\exp\left(
    %%     -n\gamma^2/2
    %%     \right)\geq \alpha
    %%     \quad  \Longleftrightarrow \quad 
        %\gamma \geq \frac{1}{\sqrt{n}}\sqrt{2\log\left(\frac{K-1}{1-\alpha}\right)}.
%\end{align*}
%For example, if $\alpha=0.95$ and $K=2$, then we require $\gamma \geq 2.4477/\sqrt{n}$. 

Our third robustness result, \Cref{thm:early_stopping}, shows that a
conservative loss leads to an early stopping property that is robust against
label noise, while a non-conservative loss generally does not have this
property and stops under a much more stringent condition.

\begin{restatable}{theorem}{es}
\label{thm:early_stopping}
    % Let $\bm{n}_{\cS}\in \mathbb{Z}_{\geq 0}^K\setminus\{\bm{0}\}$ be the vector containing the number of examples in $\cS$ from each class, 
    % and denote by 
    % $\cM_{\cS} = \{j\,:\, n_j=\|\bn_{\cS}\|_\infty\}$ 
    % the set of majority classes at $\cS$.
    % Assume tree growth is halted if 
    % $\RR(f,t;\cS)\leq 0$.
    % Then, using a conservative loss function, 
    % a node will not be split according to this criterion if and only if 
    % % $\|\bm{n}_{\cS}\|_\infty = \|\bm{n}_{\cS_{f\leq t}}\|_\infty + \|\bm{n}_{\cS_{f > t}}\|_\infty$.
    % $\cM_{\cS}=\cM_{\cS_{f\leq t}} \cap \cM_{\cS_{f > t}}$,
    % that is, the majority classes at the parent node must also be majority classes at both child nodes.  
    Assume that tree growth at a node is halted when the risk reduction at the node for any split is non-positive.
    \begin{itemize}
    \item[(a)] For a conservative loss, a node will stop splitting if and only if the majority classes at the node are also the majority classes at both child nodes for all splits.  
    \item[(b)] For the MSE, CE or GCE $(q\in(0,1))$ loss, splitting is only halted if the parent node and both child nodes all share the same label distribution for all splits.
    \end{itemize}
\end{restatable}

We found in our experiments that this early stopping phenomenon does indeed happen in practice, and tends to help tree methods avoid overfitting in situations with large amounts of noise in the training labels. 
However, we also observed that this early stopping can sometimes lead to underfitting in low noise situations, as predicted in \cite{breiman1984classification} for noise-free data.
Note that we give a necessary and sufficient condition for early stopping with
conservative loss functions and compare to each of the MSE, CE and GCE $(q\in(0,1))$, while a sufficient condition for early stopping with the misclassification impurity and comparison with Gini impurity can be found in \cite{breiman1984classification}. 
% \textcolor{red}{(4.1 and 4.3.1, respectively)}.

\section*{Distribution Losses}
\label{sec:new-framework}

We introduce a new approach for constructing robust loss functions that can
adapt to the noise level.
This allows us to address a limitation of the conservative losses as pointed out
%in \Cref{sec:theory}: 
in the previous section:
while conservative losses are robust against label noise, they can lead to
underfitting in low noise situations.

%There are three conditions that a loss function must satisfy for it to be called
%conservative and lead to the misclassification impurity in recursive greedy risk minimization. One notable condition being boundedness (condition (b) in \Cref{thm:universal}). 
%Considering loss functions satisfying only this condition will allow for discovery of new losses/impurities with some robustness without necessarily being limited to just the misclassification impurity. Help to avoid potential underfitting issues in low noise situations.

%One well known class of functions that are bounded are cumulative distribution functions (CDFs) from probability theory. 
%Recall that a function $F:\R\to\R$ is a CDF if it is right-continuous monotone increasing, $\lim_{x\to -\infty}F(x)=0$ and $\lim_{x\to \infty}F(x)=1$. 

Our approach is based on a simple idea that we will use to unify various common
losses.
Specifically, consider binary classification, and let 
$y \in \{-1, 1\}$, 
$\widehat{y} \in \R$, 
and 
$z = y\widehat{y}$ 
denote the true label, the prediction, and the margin,
respectively, then for any CDF $F$, 
$\ell(y, \widehat{y}) = F(-z)$ can be used as a loss function.
Intuitively, assume the margin of a random example follows the distribution $F$,
then the loss is the probability that the random margin is larger than a value.
We call $\ell$ a \emph{distribution loss}.
The loss is bounded in [0, 1], 
% and it converges 
converging
to 0 and 1 when $z \to +\infty$
and $z \to -\infty$, respectively.
% \textcolor{red}{convergence other way around}

Various commonly used loss functions are distribution losses.

\begin{restatable}{lemma}{distloss}
    \label{lem:distloss}
	The distribution loss is
	the 01 loss $\ell(z)=(1-\mathrm{sign}(z))/2$ 
	for the Bernoulli distribution $\ber(0)$;
	the sigmoid loss $\ell(z)=1/(1+\exp(z))$ 
	for the logistic distribution $\mathsf{Logistic}(0,1)$;
	and the ramp loss $\ell(z)=\max\{0,\min\{1,(1-z)/2\}\}$ 
	for the uniform distribution $\mathsf{U}(-1,1)$.
\end{restatable}

We instantiate the distribution loss framework to create a robust loss function
with a parameter that allows for adaptation to different noise levels.
\begin{restatable}{lemma}{neloss}
\label{lem:neloss}
	For an exponential variable $X \sim \mathsf{Exp}(1)$, consider its shifted
	negative $Z=\mu - X$ for some $\mu \ge 0$, then the CDF of $Z$ is 
	\begin{align*}
	    F(z) = \min\{1,\exp(z-\mu)\},
	\end{align*}
	and the corresponding loss function is 
	\begin{align*}
	    \ell(\widehat{y},y)=\min\{1,\exp(-y\widehat{y}-\mu)\}.
	\end{align*}
	We call this loss the negative exponential (NE) loss. 
\end{restatable}

The plot of the NE loss in \Cref{fig:neg-exp} reveals several robustness
properties:
first, the loss is capped at 1 even for large negative margins, thus preventing
imposing excessively large penalty on a noisy example far from the decision
boundary;
second, the rapid decrease of the loss to zero helps the classifier to avoid
overfitting to large positive margins;
third, the robustness parameter $\mu$ allows control on the range of negative
margins that should be penalized, with a large $\mu$ allowing ignoring more
noisy examples with negative margins.

NE loss can be viewed as a capped and shifted variant of the standard
exponential loss.
While capping creates zero gradient and thus may make gradient-based learning
difficult, this is not a limitation for decision tree learning as we only
need to compute the impurity.
\Cref{thm:negexp} gives an expression for the impurity corresponding to the NE
loss.

\begin{figure}[tb]
    \centering
    \includegraphics{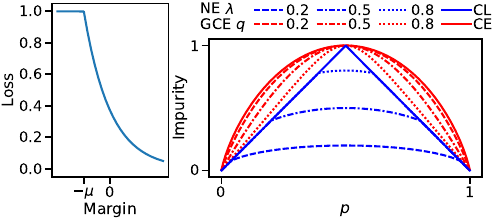}
    \caption{
    Left: NE loss as a function of the margin $y\widehat{y}$. 
	Right: 
Blue lines are NE impurities with their $\lambda$ values provided in the legend.
The special case with $\lambda=1$ is denoted by CL (conservative loss).
Red lines are GCE impurities with their $q$ values provided in the legend.
The special case with $q = 0$ is denoted by CE (cross entropy).
 % Cross entropy (also GCE impurity with $q=0$),
 %  GCE impurities for $q=0.2, 0.5, 0.8$,
 %  the misclassification impurity (also GCE impurity with $q=1$ and NE impurity with $\lambda=1$), and
 %  NE impurities for $\lambda = 0.2, 0.5, 0.8$.
  The impurities have been scaled for better comparison.
  }
    \label{fig:neg-exp}
\end{figure}

% red lines

% \textcolor{red}{Salient features: capped penalty for large margins inherited from CDF, avoid large margins by making the loss go to zero quickly enough.}

%\begin{restatable}[Negative Exponential Impurity]{theorem}{negexp}
\begin{restatable}{theorem}{negexp}
\label{thm:negexp}
	%In a binary classification problem, suppose the labels are $y\in\{-1,+1\}$. 
	The NE loss's partial empirical risk is given by
	\begin{align*}
		\widehat{R}(\widehat{y};\cS) &= \sum_{(\bx,y)\in\cS} \min\{1,\exp(-\widehat{y}y-\mu)\}/|\cD|,  
	\end{align*}
	and the corresponding impurity (i.e., minimum partial empirical risk) is
	\begin{align*}
		\underset{\widehat{y}\in\R}{\min}\, \widehat{R}(\widehat{y};\cS) 
		&= \WS \min\left\{
			1-\|\bp\|_\infty,\ 
			\lambda
			\sqrt{(1-\|\bp\|_2^2)/2}
			\right\},
	\end{align*}
	with $\bp=(p_{+},p_{-})^\T$, $p_{+}$ being the proportion of examples with $y=+1$ in $\cS$ and 
 % $\lambda=e^{-\mu}$.
 $\lambda=2e^{-\mu}$.
\end{restatable}
Note that in the definition of the NE impurity above, $\widehat{y}$ ranges over $\R$ instead 
of over the 1-simplex.
This is because we have chosen the prediction $\widehat{y}$ to be a positiveness score, instead of 
a two-dimensional vector representing a probability distribution, as done in 
the preliminaries section.
%\Cref{sec:background}.

We provide a generalization of the NE impurity to the multiclass setting:
\begin{align*}
	\underset{\widehat{y}\in\R}{\min}\, \widehat{R}(\widehat{y};\cS) 
	&= \WS \min\left\{
		1-\|\bp\|_\infty,\ 
		\lambda
		\sqrt{\frac{1-\|\bp\|_2^2}{K/(K-1)}}
		\right\}.
\end{align*}
Clearly, this reduces to the binary NE impurity when $K = 2$.
The $K/(K-1)$ factor is chosen such that when 
% $\lambda = 2^{-0} = 1$, 
$\lambda 
% = e^{-0} 
= 1$, 
the
two expressions under $\min$ have the same maximum value,
%of $(K-1)/K$,
a property that
holds for the binary case in \Cref{thm:negexp}.

We visualize the NE impurity and compare it against the 
entropy, 
GCE, % impurity, 
and 
% the 
misclassification 
impurities in \Cref{fig:neg-exp}.
The NE impurity interpolates between the misclassification impurity and the
square root of the Gini impurity:
for medium $p$ values, the NE impurity equals the square root of the Gini impurity
(up to a multiplicative constant); while for small and large $p$ values, the NE
impurity is the same as the misclassification impurity.
In particular, we get the misclassification impurity for $\lambda=1$, and we get
the square root of the Gini impurity as $\lambda \to 0$.
In contrast, the GCE impurities (including entropy) upper bound the
misclassification impurity.

The NE impurity supports an adaptive robustness mechanism:
the robustness parameter $\lambda \in (0, 1]$ controls the similarity between
the NE impurity and the misclassification impurity, and it can be tuned to adapt
to the noise rate, as done in our experiments.
A larger $\lambda$ is associated with higher similarity with the
misclassification impurity, thus encouraging more early stopping and higher
robustness.
An alternative effective way to control early stopping is setting the minimum
number of samples in a leaf \cite{ghosh2017robustness}.
Our method is distribution dependent and more flexible in the sense that it
allows for leaf nodes with varying numbers of samples. 
We also note that the GCE impurity \cite{zhang2018generalized} has a
hyperparameter $q\in[0,1]$ controlling how it interpolates between the robust
misclassification impurity and the non-robust entropy impurity. 
However, the early stopping property as in \Cref{thm:early_stopping} only takes
effect for $q \geq 1$.
We see in experiments that varying $q$ between 0 and 1 does not seem to improve
robustness to label noise for tree methods as a result. 

%The form of the minimum risk suggests a way to generalize the result to multiclass classification; by plugging in higher dimensional probability vectors $\bp$. However, for $K>2$, the form of the corresponding loss function is currently not known and it's discovery is left to future research. 

% If we consider the minimization of a scaled version of the empirical risk 
% 	$K\widehat{R}_{\NE}(\widehat{\by};\cD)/(K-1)$ 
% so that 
% $0\leq \widehat{R}_{\NE}^*(\cS)\leq 1$, 

%\paragraph{Ensemble} can use new splittig criteria in ensemble of decision trees, e.g., random forest \cite{breiman2001random} and extremely randomized trees \cite{geurts2006extremely}.
%\paragraph{Risk reduction importance} Can measure the importance of a feature using the risk reduction importance \cite{wilton2022positive}.

\section*{Experiments}
\label{sec:experiments}
% \textcolor{red}{experiments still running}

\begin{figure*}[h!]
    \centering
    \includegraphics
    % [width=\linewidth]
    [trim=0.7em 0.7em 0.7em 0.7em, clip]
    {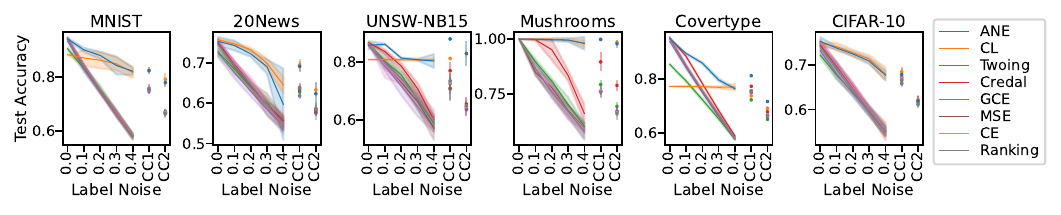}
    \caption{
    % {\scriptsize this is some demo text 0.95}
    Mean test accuracy with 2x sd bands for DT on binary classification problems using different splitting criteria.   
    Training labels corrupted using uniform noise $\eta\in\{0.0, 0.1,0.2,0.3,0.4\}$ and class conditional noise CC1 $(0.1,0.3)$ and CC2 $(0.2,0.4)$.
    % Results averaged over 5 replications with 2 standard deviations shown.
    }
    \label{fig:dt-bin}
\end{figure*}

\begin{figure*}[h!]
    \centering
    \includegraphics
    % [width=\linewidth]
    [trim=0.7em 0.7em 0.7em 0.7em, clip]
    {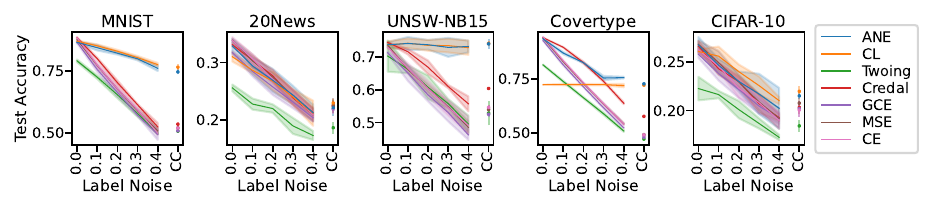}
    % \caption{Multiclass classification, decision tree, 5 replications}
    \caption{
    Mean test accuracy with 2x sd bands for DT on multiclass classification problems using different splitting criteria.   
    Training labels corrupted using uniform noise $\eta\in\{0.0,0.1,0.2,0.3,0.4\}$ and class conditional (CC) noise.
    % Results averaged over 5 replications with 2 standard deviations shown.
    }
    \label{fig:dt-mc}
\end{figure*}

\begin{figure*}[h!]
    \centering
    \includegraphics
    % [width=\linewidth]
    [trim=0.7em 0.7em 0.7em 0.7em, clip]
    {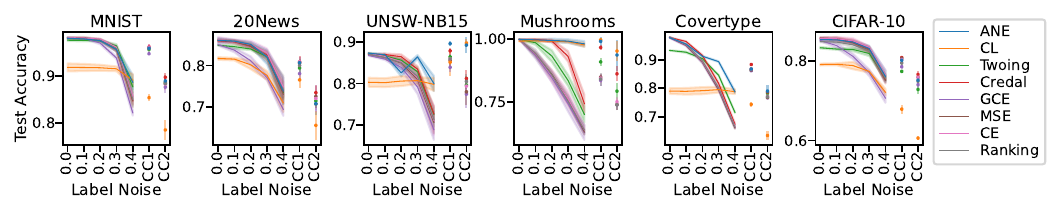}
    % \caption{Binary classification, random forest, 5 replications}
    \caption{
    Mean test accuracy with 2x sd bands for RF on binary classification problems using different splitting criteria.   
    Training labels corrupted using 
    uniform noise $\eta\in\{0.0,0.1,0.2,0.3,0.4\}$ and class conditional noise CC1 $(0.1,0.3)$ and CC2 $(0.2,0.4)$.
    % Results averaged over 5 replications with 2 standard deviations shown.
    }
    \label{fig:rf-bin}
\end{figure*}

\begin{figure*}[h!]
    \centering
    \includegraphics
    % [width=0.86\linewidth]
    [trim=0.7em 0.7em 0.7em 0.7em, clip]
    {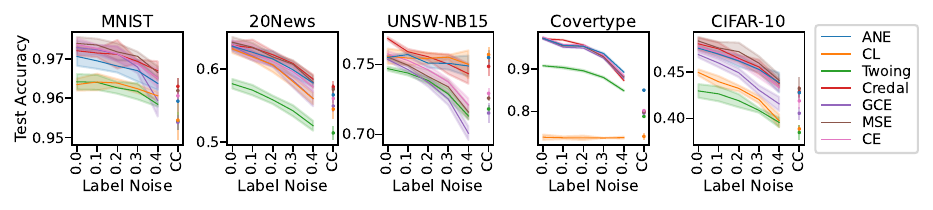}
    % \caption{Multiclass classification, random forest, 5 replications}
    \caption{
    Mean test accuracy with 2x sd bands for RF on multiclass classification problems using different splitting criteria.   
    Training labels corrupted using uniform noise $\eta\in\{0.0,0.1,0.2,0.3,0.4\}$ and class conditional (CC) noise.
    % Uniform noise levels $0.0$, $0.1$, $0.2$, $0.3$, $0.4$ and class conditional (CC) noise.
    % Results averaged over 5 replications with 2 standard deviations shown.
    }
    \label{fig:rf-mc}
\end{figure*}

In this section we empirically compare the NE loss with several other loss functions and tree growing methods. Selected baselines include 
NE loss, 
conservative losses (CLs), 
twoing split criteria \cite{breiman1984classification}, 
Credal-C4.5 \cite{mantas2014credal}, 
GCE loss \cite{zhang2018generalized}, 
ranking loss \cite{yang2019robust}, 
CE loss and MSE loss.
Note that twoing, Gini and misclassification are shown to be robust with large number of samples at leaf in \cite{ghosh2017robustness}.
We are interested in the predictive performance on clean test data of both
decision trees and random forests trained using noisily labeled training data in
various label noise settings. 
We also investigate the effect of tuning the hyperparameter $\lambda$ in the NE impurity. 
%We tune the hyperparameter $\lambda$ for the NE impurity to allow adapting to
%the noise rate.

\paragraph{Datasets} 
We consider some commonly used datasets from UCI including Covertype \cite{misc_covertype_31}, 20News \cite{lang1995newsweeder}, Mushrooms \cite{misc_mushroom_73}, as well as the MNIST digits \cite{lecun1998gradient}, CIFAR-10 \cite{krizhevsky2009learning} and UNSW-NB15 \cite{moustafa2015unsw}. 
Datasets were selected due to diversity in number of samples, number of features, number of classes, type (image, tabular, text), and domain (computer vision, cyber security, cartography).
We look at both multiclass ($K>2$) and binary ($K=2$) classification problems for each dataset. Mushrooms datasets and ranking loss excluded from multiclass classification experiments due to Mushrooms being a binary classification dataset, and ranking loss only defined for binary classification problems. 
\Cref{tab:datasets} is a summary of the benchmark datasets.

\begin{table}[t]
    \centering
    \begin{tabular}{lcccc}
        Name & Train & Test & Feature & Class  \\\hline
         MNIST Digits & 60 000 & 10 000 & 784 & 10 \\
         CIFAR-10 & 50 000 & 10 000 & 3 072 & 10 \\
         20News & 11 313 & 7 531 & 300 & 20 \\
         UNSW-NB15 & 175 341 & 82 332 & 39 & 10 \\
         Covertype & 464 809 & 116 203 & 54 & 7 \\
         Mushrooms & 6499 & 1625 & 112 & 2
    \end{tabular}
    \caption{Benchmark datasets.}
    \label{tab:datasets}
\end{table}

% \paragraph{preprocesssing} 
Binarized versions of labels are based on the processing in \cite{kiryo2017positive,wilton2022positive}.
For 20News, GloVe pre-trained word embeddings \cite{pennington2014glove} were used (glove.840B.300d), 
then average pooling was applied over the word embeddings to generate the embedding for each document. 
The binarized classes are ‘alt., comp., misc., rec.’ versus ‘sci., soc., talk.’.
For MNIST, the binarized classes are  ‘0, 2, 4, 6, 8’ versus ‘1, 3, 5, 7, 9’ (even vs odd).
For CIFAR-10, the binarized classes are ‘airplane, automobile, ship, truck’ versus ‘bird, cat, deer, dog, frog, horse’ (animal vs non-animal).
For UNSW-NB15, we removed \texttt{ID} and the nominal features \texttt{proto}, \texttt{service} and \texttt{state}. 
The binarized labels are attack versus benign.
For Covertype, the binarized classes are the second class versus others, as done in \cite{collobert2001parallel}, and the train-test split was performed using 
scikit-learn train\_test\_split with train\_size 0.8 and random\_state 0.
For Mushrooms, we used the LIBSVM \cite{chang2011libsvm} version, and the train-test split was performed identically as for Covertype. 

\paragraph{Label Noise} 
We used both uniform noise and class conditional noise to corrupt the training labels, while testing labels had no noise applied. 
For uniform label noise, noise rates $\eta=0.0,\ 0.1,\ 0.2,\ 0.3$ and $0.4$ were used. 
For class conditional noise in binary classification two noise rates were used; $(0.1,0.3)$ and $(0.2,0.4)$. 
For multiclass classification, transition probabilities were constructed based on the similarity between classes using the Mahalanobis distance. 
See 
% \Cref{alg:ccn} 
Algorithm 1
in 
% \Cref{app:class-conditional-noise} 
Appendix J
for details.

\paragraph{Hardware \& Software} 
Experiments were performed using Python on a computer cluster with Intel Xeon E5 Family CPUs @ 2.20GHz and 192GB memory running CentOS. 

\paragraph{Hyperparameters}
Following common practice \cite{pedregosa2011scikit}, 
we set no restriction on maximum depth or number of leaf nodes, 
minimum one sample per leaf node,
100 trees in each RF, 
each using only $\sqrt{d}$ randomly chosen features, 
%trained on bootstrapped training set of size $n$,
and 
predictions with RF made by majority average label distribution over all trees.
The NE impurity hyperparameter $\lambda\in\{0,0.25,0.5,0.75,1\}$ was tuned by
training on 80\% of the noisy training set and validation on the other 20\%.
GCE $q=0.7$ as recommended for NNs \cite{zhang2018generalized},
Credal-C4.5 $s=1$ as recommended in \cite{mantas2014credal}. 

\subsection{Classification Performance}
We trained each model on 5 random noisy training sets to account for randomness in
training labels and learning algorithms.
We report the average test set accuracies and 2x standard deviation bands. 
% in both the main text and in \Cref{app:results}. 
% We report the average test set accuracies in the main text, but additionally
% report 2x standard deviation bands in \Cref{app:results}.
%Our theoretical results apply directly to DT, however we also test RF versions 
Additional results can be found in 
% \Cref{app:results,app:cifar10-pretrained}.
Appendices K and L.

\paragraph{Effect of Loss Function on Robustness to Label Noise for Decision Trees}
Decision tree results for binary and multiclass settings are summarized in
\Cref{fig:dt-bin,fig:dt-mc}, respectively.
Adaptive NE loss (ANE; NE with tuned $\lambda$) is almost always a top performer
across all label noise settings, label configurations (binary, multiclass) and
datasets. 
Conservative loss functions sometimes give poor performance in low noise
settings, but usually give strong performance in high noise settings. 
CE, MSE, Ranking, GCE give similar performance. 
Twoing split criteria typically gives the worst performance. 
Credal-C4.5, across many noise settings and datasets, usually has performance
somewhere between CE, MSE, Ranking, GCE and ANE.

\paragraph{Effect of Loss Function on Robustness to Label Noise for Random Forest}
Random forest results for binary and multiclass settings are summarized in
\Cref{fig:rf-bin,fig:rf-mc}, respectively.
% Some experiments exceeded the allocated time limit, particularly those for twoing, but it does not seem that these affect the qualitative conclusions below.
We first note that the 
% performance differences for different loss functions 
differences in performance across loss functions
are less significant than for
decision trees in general, suggesting that an ensemble is more robust than a
single tree.
Note that RF still seems to boost performance of non-conservative loss functions,
though this boost is less for conservative losses. 
Overall, ANE still shows strong performance. 

\begin{figure}[t]
    \centering
    \includegraphics{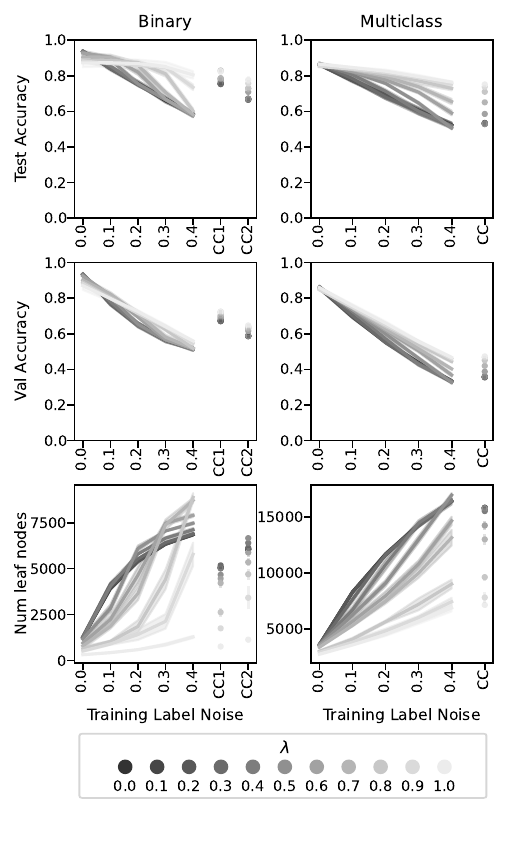}
    \caption{Performance of DT using NE loss with different values of $\lambda$ on MNIST digits, binary and multiclass classification, 20\% of noisily labeled training data used for validation.
    Results reported as mean $\pm$ 2x sd over 5 replications.}
    \label{fig:threshold_tuning}
\end{figure}

\paragraph{Effect of NE Threshold on Predictive Performance and Tree Size}
We investigated the effect of the robustness parameter $\lambda$ for the NE
impurity by performing experiments on 
% both the binary and multiclass versions of
the MNIST digits dataset. 
%Each experiment replicated 5 times to account for randomness in training labels.
%20\% noisy training labels reserved for validation. 
The results in \Cref{fig:threshold_tuning} show that different $\lambda$ values
often have best performance for different noise settings, without any single
$\lambda$ value yielding the best performance in all cases.
This together with the results in \Cref{fig:dt-bin,fig:dt-mc} suggests that
tuning $\lambda$ allows effective control on early stopping, 
and in turn,
robustness to label noise.  
Additional discussion on early stopping in binary classification is provided in 
% \Cref{app:alternate_early_stopping}.
Appendix M.
% \Cref{app:alternate_early_stopping} contains additional discussion on early stopping.
%Method with the best validation accuracy on noisy data does seem to also perform well on noise free data.
%Conservative loss often the best choice, particularly for high levels of label noise. 
%Sqrt gini seldom the best performer.

% \textcolor{red}{If reviewers ask about contradictions in results from competing methods: 
% other papers do not disentangle effect of loss function on predictive performance, often other parameters such as min samples at leaf are tuned, and max depth specified. Both of these factors affect robustness to label noise. }

\section*{Conclusions}

We developed new insight and tools on robust loss functions for decision tree learning.
We introduced the conservative loss as a general class of robust loss functions
and provided several results showing their robustness in decision tree learning.
In addition, we introduced the distribution loss as a general framework for
building loss functions based CDFs, and instantiated it to introduce the
robust NE loss.
Our experiments demonstrated that our NE loss effectively alleviates the adverse effect of 
noise in various noise settings.

%future work:
%\begin{itemize}
    %% \item also tuning min samples leaf
    %\item explanations for robustness of ensembles of trees to label noise, 
    %\item model size vs performance with noisy labels using ANE,
    %\item Impurities based on other robust loss functions 
    %\item effectively incorporating unlabeled data into training. 
%\end{itemize}

% \section*{Ethical Statement}
% optional 

\section*{Acknowledgments}
We thank the anonymous reviewers for their many helpful comments, which we have incorporated to 
significantly improve the presentation of the paper.
In particular, we thank one of them for a comment that inspired the discussion in
% \Cref{app:alternate_early_stopping}.
Appendix M.

% \bibliography{aaai24}
\bibliography{refs}

\begin{thebibliography}{37}
\providecommand{\natexlab}[1]{#1}

\bibitem[{{Audubon Society Field Guide}(1987)}]{misc_mushroom_73}
{Audubon Society Field Guide}. 1987.
\newblock {Mushroom}.
\newblock UCI Machine Learning Repository.
\newblock {DOI}: https://doi.org/10.24432/C5959T.

\bibitem[{Blackard(1998)}]{misc_covertype_31}
Blackard, J. 1998.
\newblock {Covertype}.
\newblock UCI Machine Learning Repository.
\newblock {DOI}: https://doi.org/10.24432/C50K5N.

\bibitem[{Breiman(2001)}]{breiman2001random}
Breiman, L. 2001.
\newblock Random forests.
\newblock \emph{Machine learning}, 45: 5--32.

\bibitem[{Breiman et~al.(1984)Breiman, Friedman, Olshen, and Stone}]{breiman1984classification}
Breiman, L.; Friedman, J.~H.; Olshen, R.~A.; and Stone, C.~J. 1984.
\newblock \emph{Classification and regression trees}.
\newblock Wadsworth statistics/probability series. Belmont, Calif.: Wadsworth International Group.
\newblock ISBN 0534980538.

\bibitem[{Brodley, Friedl et~al.(1996)}]{brodley1996identifying}
Brodley, C.~E.; Friedl, M.~A.; et~al. 1996.
\newblock Identifying and eliminating mislabeled training instances.
\newblock In \emph{Proceedings of the National Conference on Artificial Intelligence}, 799--805.

\bibitem[{Chang and Lin(2011)}]{chang2011libsvm}
Chang, C.-C.; and Lin, C.-J. 2011.
\newblock LIBSVM: a library for support vector machines.
\newblock \emph{ACM transactions on intelligent systems and technology (TIST)}, 2(3): 1--27.

\bibitem[{Collobert, Bengio, and Bengio(2001)}]{collobert2001parallel}
Collobert, R.; Bengio, S.; and Bengio, Y. 2001.
\newblock A parallel mixture of SVMs for very large scale problems.
\newblock \emph{Advances in Neural Information Processing Systems}, 14.

\bibitem[{Fr{\'e}nay and Verleysen(2013)}]{frenay2013classification}
Fr{\'e}nay, B.; and Verleysen, M. 2013.
\newblock Classification in the presence of label noise: a survey.
\newblock \emph{IEEE transactions on neural networks and learning systems}, 25(5): 845--869.

\bibitem[{Geurts, Ernst, and Wehenkel(2006)}]{geurts2006extremely}
Geurts, P.; Ernst, D.; and Wehenkel, L. 2006.
\newblock Extremely randomized trees.
\newblock \emph{Machine learning}, 63: 3--42.

\bibitem[{Ghosh, Kumar, and Sastry(2017)}]{ghosh2017robust}
Ghosh, A.; Kumar, H.; and Sastry, P.~S. 2017.
\newblock Robust Loss Functions under Label Noise for Deep Neural Networks.
\newblock \emph{Proceedings of the AAAI Conference on Artificial Intelligence}, 31(1).

\bibitem[{Ghosh, Manwani, and Sastry(2015)}]{ghosh2015making}
Ghosh, A.; Manwani, N.; and Sastry, P. 2015.
\newblock Making risk minimization tolerant to label noise.
\newblock \emph{Neurocomputing}, 160: 93--107.

\bibitem[{Ghosh, Manwani, and Sastry(2017)}]{ghosh2017robustness}
Ghosh, A.; Manwani, N.; and Sastry, P. 2017.
\newblock On the robustness of decision tree learning under label noise.
\newblock In \emph{Advances in Knowledge Discovery and Data Mining: 21st Pacific-Asia Conference, PAKDD 2017, Jeju, South Korea, May 23-26, 2017, Proceedings, Part I 21}, 685--697. Springer.

\bibitem[{Grinsztajn, Oyallon, and Varoquaux(2022)}]{grinsztajn2022tree}
Grinsztajn, L.; Oyallon, E.; and Varoquaux, G. 2022.
\newblock Why do tree-based models still outperform deep learning on typical tabular data?
\newblock \emph{Advances in Neural Information Processing Systems}, 35: 507--520.

\bibitem[{Kaggle(2021)}]{kaggle2021state}
Kaggle. 2021.
\newblock State of Machine Learning and Data Science 2021.

\bibitem[{Kim et~al.(2019)Kim, Yim, Yun, and Kim}]{kim2019nlnl}
Kim, Y.; Yim, J.; Yun, J.; and Kim, J. 2019.
\newblock Nlnl: Negative learning for noisy labels.
\newblock In \emph{Proceedings of the IEEE/CVF international conference on computer vision}, 101--110.

\bibitem[{Kiryo et~al.(2017)Kiryo, Niu, Du~Plessis, and Sugiyama}]{kiryo2017positive}
Kiryo, R.; Niu, G.; Du~Plessis, M.~C.; and Sugiyama, M. 2017.
\newblock Positive-unlabeled learning with non-negative risk estimator.
\newblock \emph{Advances in neural information processing systems}, 30.

\bibitem[{Krizhevsky(2009)}]{krizhevsky2009learning}
Krizhevsky, A. 2009.
\newblock Learning multiple layers of features from tiny images.
\newblock MSc thesis, University of Toronto.

\bibitem[{Lang(1995)}]{lang1995newsweeder}
Lang, K. 1995.
\newblock Newsweeder: Learning to filter netnews.
\newblock In \emph{Machine learning proceedings 1995}, 331--339. Elsevier.

\bibitem[{LeCun et~al.(1998)LeCun, Bottou, Bengio, and Haffner}]{lecun1998gradient}
LeCun, Y.; Bottou, L.; Bengio, Y.; and Haffner, P. 1998.
\newblock Gradient-based learning applied to document recognition.
\newblock \emph{Proceedings of the IEEE}, 86(11): 2278--2324.

\bibitem[{Lukasik et~al.(2020)Lukasik, Bhojanapalli, Menon, and Kumar}]{lukasik2020does}
Lukasik, M.; Bhojanapalli, S.; Menon, A.; and Kumar, S. 2020.
\newblock {Does label smoothing mitigate label noise?}
\newblock In \emph{International Conference on Machine Learning}, 6448--6458. PMLR.

\bibitem[{Lyu and Tsang(2020)}]{Lyu2020Curriculum}
Lyu, Y.; and Tsang, I.~W. 2020.
\newblock Curriculum Loss: Robust Learning and Generalization against Label Corruption.
\newblock In \emph{International Conference on Learning Representations}.

\bibitem[{Ma et~al.(2020)Ma, Huang, Wang, Romano, Erfani, and Bailey}]{ma2020normalized}
Ma, X.; Huang, H.; Wang, Y.; Romano, S.; Erfani, S.; and Bailey, J. 2020.
\newblock Normalized loss functions for deep learning with noisy labels.
\newblock In \emph{International conference on machine learning}, 6543--6553. PMLR.

\bibitem[{Mantas and Abellan(2014)}]{mantas2014credal}
Mantas, C.~J.; and Abellan, J. 2014.
\newblock Credal-C4. 5: Decision tree based on imprecise probabilities to classify noisy data.
\newblock \emph{Expert Systems with Applications}, 41(10): 4625--4637.

\bibitem[{Manwani and Sastry(2013)}]{manwani2013noise}
Manwani, N.; and Sastry, P. 2013.
\newblock Noise tolerance under risk minimization.
\newblock \emph{IEEE transactions on cybernetics}, 43(3): 1146--1151.

\bibitem[{Moustafa and Slay(2015)}]{moustafa2015unsw}
Moustafa, N.; and Slay, J. 2015.
\newblock UNSW-NB15: a comprehensive data set for network intrusion detection systems (UNSW-NB15 network data set).
\newblock In \emph{2015 military communications and information systems conference (MilCIS)}, 1--6. IEEE.

\bibitem[{Natarajan et~al.(2013)Natarajan, Dhillon, Ravikumar, and Tewari}]{natarajan2013learning}
Natarajan, N.; Dhillon, I.~S.; Ravikumar, P.~K.; and Tewari, A. 2013.
\newblock Learning with noisy labels.
\newblock \emph{Advances in neural information processing systems}, 26.

\bibitem[{Painsky and Wornell(2018)}]{painsky2018universality}
Painsky, A.; and Wornell, G. 2018.
\newblock On the universality of the logistic loss function.
\newblock In \emph{2018 IEEE International Symposium on Information Theory (ISIT)}, 936--940. IEEE.

\bibitem[{Patrini et~al.(2017)Patrini, Rozza, Krishna~Menon, Nock, and Qu}]{patrini2017making}
Patrini, G.; Rozza, A.; Krishna~Menon, A.; Nock, R.; and Qu, L. 2017.
\newblock Making deep neural networks robust to label noise: A loss correction approach.
\newblock In \emph{Proceedings of the IEEE conference on computer vision and pattern recognition}, 1944--1952.

\bibitem[{Pedregosa et~al.(2011)Pedregosa, Varoquaux, Gramfort, Michel, Thirion, Grisel, Blondel, Prettenhofer, Weiss, and Dubourg}]{pedregosa2011scikit}
Pedregosa, F.; Varoquaux, G.; Gramfort, A.; Michel, V.; Thirion, B.; Grisel, O.; Blondel, M.; Prettenhofer, P.; Weiss, R.; and Dubourg, V. 2011.
\newblock Scikit-learn: Machine learning in Python.
\newblock \emph{the Journal of machine Learning research}, 12: 2825--2830.

\bibitem[{Pennington, Socher, and Manning(2014)}]{pennington2014glove}
Pennington, J.; Socher, R.; and Manning, C.~D. 2014.
\newblock GloVe: Global Vectors for Word Representation.
\newblock In \emph{Empirical Methods in Natural Language Processing (EMNLP)}, 1532--1543.

\bibitem[{Song et~al.(2022)Song, Kim, Park, Shin, and Lee}]{song2022learning}
Song, H.; Kim, M.; Park, D.; Shin, Y.; and Lee, J.-G. 2022.
\newblock Learning from noisy labels with deep neural networks: A survey.
\newblock \emph{IEEE Transactions on Neural Networks and Learning Systems}.

\bibitem[{Tanno et~al.(2019)Tanno, Saeedi, Sankaranarayanan, Alexander, and Silberman}]{tanno2019learning}
Tanno, R.; Saeedi, A.; Sankaranarayanan, S.; Alexander, D.~C.; and Silberman, N. 2019.
\newblock {Learning from noisy labels by regularized estimation of annotator confusion}.
\newblock In \emph{Proceedings of the IEEE/CVF conference on computer vision and pattern recognition}, 11244--11253.

\bibitem[{Wang et~al.(2019)Wang, Ma, Chen, Luo, Yi, and Bailey}]{wang2019symmetric}
Wang, Y.; Ma, X.; Chen, Z.; Luo, Y.; Yi, J.; and Bailey, J. 2019.
\newblock Symmetric cross entropy for robust learning with noisy labels.
\newblock In \emph{Proceedings of the IEEE/CVF international conference on computer vision}, 322--330.

\bibitem[{Wilton et~al.(2022)Wilton, Koay, Ko, Xu, and Ye}]{wilton2022positive}
Wilton, J.; Koay, A.; Ko, R.; Xu, M.; and Ye, N. 2022.
\newblock Positive-Unlabeled Learning using Random Forests via Recursive Greedy Risk Minimization.
\newblock \emph{Advances in Neural Information Processing Systems}, 35: 24060--24071.

\bibitem[{Yang, Gao, and Li(2019)}]{yang2019robust}
Yang, B.-B.; Gao, W.; and Li, M. 2019.
\newblock On the robust splitting criterion of random forest.
\newblock In \emph{2019 IEEE International Conference on Data Mining (ICDM)}, 1420--1425. IEEE.

\bibitem[{Zhang and Sabuncu(2018)}]{zhang2018generalized}
Zhang, Z.; and Sabuncu, M. 2018.
\newblock Generalized cross entropy loss for training deep neural networks with noisy labels.
\newblock \emph{Advances in neural information processing systems}, 31.

\bibitem[{Zhou, Ding, and Li(2019)}]{zhou2019improving}
Zhou, X.; Ding, P. L.~K.; and Li, B. 2019.
\newblock Improving robustness of random forest under label noise.
\newblock In \emph{2019 IEEE Winter Conference on Applications of Computer Vision (WACV)}, 950--958. IEEE.

\end{thebibliography}

\clearpage
\appendix
\section{Proof of \Cref{thm:loss_impurity_equivalence}}
\label{app:equivalence_thm}
\lie*
\begin{proof} 
    \noindent{\bf (a)} 
    The MSE is
    \begin{align*}
        \widehat{R}_{\mathrm{MSE}}(\widehat{\by};\cS) &= \sum_{(\bx,\by)\in\cS} \|\widehat{\by}-\by\|_2^2/|\cD|.
    \end{align*}
    Setting the gradient to 0, 
    \begin{align*}
        \nabla_{\widehat{\by}} \widehat{R}_{\mathrm{MSE}}(\widehat{\by};\cS)
        =  \sum_{(\bx,\by)\in\cS} \frac{2(\widehat{\by}-\by)}{|\cD|} 
        = 2W_{\cS}(\widehat{\by}-\bp)
        = 0.
    \end{align*}
    Solving the equation gives
    \begin{align*}
        \widehat{\by}^* &= \bp.
    \end{align*}
    Hence, knowing that $\bp$ and each of $\{\by\}$ are probability vectors, the minimum partial empirical risk is
    \begin{align*}
        \widehat{R}_{\mathrm{MSE}}^*(\cS) &= \sum_{(\bx,\by)\in\cS} \| \bp-\by \|_2^2/|\cD|,
        \\
        &= \frac{1}{|\cD|} \sum_{(\bx,\by)\in\cS} \left[\|\bp\|_2^2 + 1 - 2\bp^\T\by\right],
        % \\
        % &= W_{\cS}\|\bp\|_2^2 + W_{\cS} - 2w\bp^\T \sum_{(\bx,\by)\in\cS} \by
        \\
        &= W_{\cS}\|\bp\|_2^2 + W_{\cS} - 2W_{\cS}\bp^\T \bp,
        % \\
        % &= \|\bp\|_2^2 + 1 - 2\|\bp\|_2^2
        \\
        &= W_{\cS}(1-\|\bp\|_2^2).
    \end{align*}

    \medskip\noindent{\bf (b)} 
    For the CE loss, the partial empirical risk can be simplified by grouping together examples with common labels
    \begin{align*}
        \widehat{R}_{\mathrm{CE}}(\widehat{\by};\cS) &=  -\frac{1}{|\cD|}\sum_{(\bx,\by)\in\cS} \by^\T \log \widehat{\by},
        \\
        &= -W_{\cS} \sum_{j=1}^K p_j\log \widehat{y}_j.
    \end{align*}
    Use Lagrange multiplier to ensure solution sums to one
    \begin{align*}
        \mathcal{L}(\widehat{\by},\lambda) &:= -W_{\cS} \sum_{j=1}^K p_j\log \widehat{y}_j + \lambda \left( \sum_{j=1}^K \widehat{y}_j - 1 \right).
    \end{align*}
    The $k$-th component of the optimal constant probability vector prediction satisfies the equation
    \begin{align*}
        \frac{\partial\mathcal{L}}{\partial\widehat{y}_k} = -W_{\cS}  p_k/\widehat{y}_k + \lambda=0.
    \end{align*}
    Solving this equation gives $\widehat{y}_k^* = \WS p_k/\lambda$. Applying the constraint to solve for $\lambda$ gives
    \begin{align*}
        1 &= \sum_{j=1}^K \widehat{y}_j^* = \sum_{j=1}^K \WS p_k/\lambda = \WS/\lambda,
        \\
        \lambda &= \WS.
    \end{align*}
    Plugging the optimal $\lambda$ into the expression for $\widehat{y}_j^*$, we obtain
    \begin{align*}
        \widehat{\by}^* &= \bp,
        \\
        \widehat{R}_{\mathrm{CE}}^*(\cS) &= \WS(-\bp^\T\log \bp).
    \end{align*}

    \medskip\noindent{\bf (c)} 
    The form of the partial empirical risk can be simplified based on the prediction $\widehat{\by}$
    \begin{align*}
        \widehat{R}_{01}(\widehat{\by};\cS) &= \frac{1}{|\cD|}\sum_{(\bx,\by)\in\cS}  \I(\widehat{\by}\ne\by),
        \\
        &= \WS \sum_{j=1}^K p_j \I(\widehat{\by}\ne \be_j),
        \\
        &= \WS\begin{cases}
            1,& \widehat{\by} \notin \{\be_j\}_{j=1}^K, \\
            1-p_j,& \widehat{\by}=\be_j.
        \end{cases}  
    \end{align*}
    The prediction $\widehat{\by}^*$ that minimizes the partial empirical risk is thus   
    \begin{align*}
        \widehat{\by}^* &= \bm{e}_{\argmax(\bp)},
    \end{align*}
    with optimal risk
    \begin{align*}
        \widehat{R}_{01}^*(\cS) &= \WS(1-\|\bp\|_\infty).
    \end{align*}

    \medskip\noindent{\bf (d)} 
    % \label{sec:GCE_q_in_01}
    For $q=0$, the GCE loss is identical to CE loss \cite{zhang2018generalized}, and hence both loss functions lead to the same impurity. 
    
    Fix $q\in(0,1)$. The partial empirical GCE risk with constant prediction $\widehat{\by}$ is
    \begin{align*}
        % \ell(\widehat{\by},\by) &= \frac{1-(\by^\T\widehat{\by})^q}{q}
        % \\
        \widehat{R}_{\mathrm{GCE}}(\widehat{\by};\cS) &= \frac{1}{|\cD|} \sum_{(\bx,\by)\in\cS} \frac{1-(\by^\T\widehat{\by})^q}{q},
        \\
        &= \WS\sum_{j=1}^K p_j \frac{1-\widehat{y}_j^q}{q},
        \\
        &= \WS (
        1-\sum_{j=1}^K p_j \widehat{y}_j^q
        )/q.
    \end{align*}
    By H\"{o}lder's inequality we have 
    \begin{align*}
        \sum_{j=1}^K p_j\widehat{y}_j^q 
        &\leq \|\bp\|_{1/(1-q)} \left( \sum_{j=1}^K (\widehat{y}_j^{q})^{1/q} \right)^{q}, 
        \\
        &= \|\bp\|_{1/(1-q)} \left( \sum_{j=1}^K \widehat{y}_j \right)^{q}. 
    \end{align*}
    We get equality in H\"older's inequality if and only if 
    $\widehat{y}_j^q \propto p_j^{1/(1-q)-1},\,\forall j=1,\ldots,K$. 
    Therefore, we can choose $\widehat{\by}^*$ to be the probability vector with elements 
    $\widehat{y}_j^* = p_j^{1/(1-q)}/\sum_{k=1}^K p_k^{1/(1-q)},\,\forall j=1,\ldots,K$. Hence, the minimum partial empirical risk becomes
    % \textcolor{red}{Use holders inequatilty, not lagrange multipliers}    
    % Apply Lagrange multiplier to ensure solution sums to one:
    % \begin{align*}
    %     {\cal{L}}(\widehat{\by},\lambda) :&= \WS\left(
    %     \frac{1}{q}-\frac{1}{q}\sum_{j=1}^K p_j \widehat{y}_j^q \right)
    %     +\lambda \left( \sum_{j=1}^K \widehat{y}_j-1 \right)
    %     \\
    %     \frac{\text{d}{\cal{L}}}{\text{d} \widehat{y}_k} &= -\WS p_k\widehat{y}_k^{q-1} +\lambda
    %     \\
    %     \widehat{y}_k^* &= \left( \frac{\lambda}{p_k\WS} \right)^{1/(q-1)}
    %     \\
    %     &= \WS^{1/(1-q)} \lambda^{1/(q-1)} p_k^{1/(1-q)} 
    % \end{align*}
    % Apply constraint to find $\lambda$:
    % \begin{align*}
    %     1 &= \sum_j \widehat{y}_j^* 
    %     \\
    %     &= \WS^{1/(1-q)}\lambda^{1/(q-1)} \sum_j  p_j^{1/(1-q)}
    %     \\
    %     \lambda^* &= \frac{\WS}{ \left(\sum_j  p_j^{1/(1-q)}\right)^{q-1} }
    %     \\
    %     \widehat{y}_k^* &= \WS^{1/(1-q)} \left( \frac{\WS}{ \left(\sum_j  p_j^{1/(1-q)}\right)^{q-1} } \right)^{1/(q-1)} p_k^{1/(1-q)}
    %     \\
    %     &= \frac{p_k^{1/(1-q)}}{\sum_j p_j^{1/(1-q)}}
    % \end{align*}
    % The minimum partial risk is:
    \begin{align*}
        \widehat{R}_{\mathrm{GCE}}^*(\cS) 
        % &= \WS\left(\frac{1}{q}-\frac{1}{q}\sum_j p_j \widehat{y}_j^{*q}\right)
        % \\
        % &= \WS\left(\frac{1}{q}-\frac{1}{q}\sum_j p_j \left(\frac{p_j^{1/(1-q)}}{\sum_k p_k^{1/(1-q)}}\right)^q\right)
        % \\
        % &= \frac{1}{q}-\frac{1}{q}\left(\sum_k p_k^{1/(1-q)}\right)^{-q}\sum_j p_j^{1/(1-q)}
        % \\
        % &= \frac{1}{q}-\frac{1}{q}\left(\sum_k p_k^{1/(1-q)}\right)^{1-q}
        % \\
        &= \WS(1-\|\bp\|_{1/(1-q)})/q.
    \end{align*}
    % \vspace{.2cm}
    Now let $q>1$. Similarly as for the case $q\in(0,1)$ we can simplify by grouping together examples from each class
    \begin{align*}
        \widehat{R}_{\mathrm{GCE}}(\widehat{\by};\cS) 
        % &= \sum_{(\bx,\by)\in\cS} w\frac{1-(\by^\T\widehat{\by})^q}{q}, 
        % \\
        &= \WS(1-\sum_{j=1}^k p_j\widehat{y}_j^q)/q.
    \end{align*}
    Using $\widehat{\by} \in \simplex$ and $q>1$, we have $\widehat{y}_j^q\leq \widehat{y}_j\,\forall j=1,\ldots,K$ and
    % , 
    % similarly to the argument used for MAE, 
    $\sum_{j=1}^K p_j\widehat{y}_j\leq \max_{j}p_j=\|\bp\|_\infty$. Combining these facts results in 
    \begin{align*}
        \WS(1-\sum_{j=1}^k p_j\widehat{y}_j^q)/q 
        &\geq \WS(1-\sum_{j=1}^k p_j\widehat{y}_j)/q, 
        \\
        &\geq \WS(1-\|\bp\|_\infty )/q.%\label{eq:GCE_q_geq_1} 
    \end{align*}
    We have equality in the above for $\widehat{\by}^*=\be_{\argmax(\bp)}$, giving optimal objective value
    \begin{align*}
        \widehat{R}_{\mathrm{GCE}}^*(\cS) &= \WS(1-\|\bp\|_\infty)/q.
    \end{align*}
    
    \medskip\noindent{\bf (e)} 
    Using the fact that $\widehat{\by} \in \simplex$ and $\by\in\{\be_j\}_{j=1}^K$ for each $\by\in\cS$, the partial empirical risk can be simplified as
    \begin{align*}
        \widehat{R}_{\mathrm{MAE}}(\widehat{\by};\cS) &=  \frac{1}{|\cD|}\sum_{(\bx,\by)\in\cS} \|\widehat{\by}-\by\|_1,
        \\
        &= \frac{1}{|\cD|} \sum_{(\bx,\by)\in\cS} 2-2\widehat{\by}^\T\by,
        \\
        &= 2\WS(1-\widehat{\by}^\T\bp).
    \end{align*}
    Since $\widehat{\by}^\T\bp=\sum_{j=1}^K \widehat{y}_jp_j$ is a convex combination of the values $\{p_j\}$ with weights $\{\widehat{y}_j\}$, we have $\widehat{\by}^\T\bp\leq \max_{j} p_j=\|\bp\|_\infty$, with equality achieved for 
    \begin{align*}
        \widehat{\by}^* &= \be_{\argmax(\bp)}.
    \end{align*}
    Plugging the optimal prediction into the objective gives the result
    \begin{align*}
        \widehat{R}_{\mathrm{MAE}}^*(\cS) &= 2\WS(1-\|\bp\|_\infty). \qedhere
    \end{align*}
\end{proof}

\section{Proof of \Cref{thm:universal}}
\label{app:proof_universal}
The proof of \Cref{thm:universal} relies on \Cref{lemma:earth-moving}.

% \begin{lemma}
\begin{restatable}[]{lemma}{earthmoving}
\label{lemma:earth-moving} 
Let $(p_1,\ldots,p_K) \in \simplex$ be a probability vector and $\{\ell_j\}_{j=1}^K\subset \R$ be a set of constants. If 
\begin{enumerate}
    \item[(a)] $p_1\leq \cdots\leq p_K$, 
    \item[(b)] $\sum_{j=1}^K \ell_j\geq C(K-1)$ and 
    \item[(c)] $0\leq \ell_j\leq C,\,\forall j=1,\ldots,K$, 
\end{enumerate}
then
\begin{align}
    \sum_{j=1}^K p_j l_j &\geq C(1-p_K). \label{eq:earthmoving}
\end{align}
% \end{lemma}
\end{restatable}

\begin{proof}
    Assume each of $(a)$, $(b)$ and $(c)$ are true. 
    Starting with the right hand side of \eqref{eq:earthmoving}, we can write
    \begin{align*}
        & C(1-p_K) = \sum_{j\ne K} p_j C
        \\
        = & \sum_{j\ne K} p_j (\ell_j + (C-\ell_j)) 
        + p_K 
        (
            \ell_K + \sum_{j\ne K}\ell_j - \sum_{j=1}^K\ell_j
        ), 
        \\
        \leq & \sum_{j\ne K} p_j (\ell_j + (C-\ell_j)) 
        \\
        &\quad + p_K 
        (
            \ell_K + \sum_{j\ne K}\ell_j - C(K-1)
        ),
        % ,\quad (\text{Using assumption }(b))
        \\
        =& \sum_{j\ne K} p_j (\ell_j + (C-\ell_j)) + p_K (\ell_K - \sum_{j\ne K}(C-\ell_j)),  
        \\
        =& \sum_{j\ne K} p_j \ell_j + \sum_{j\ne K} p_j(C-\ell_j) 
        \\
        &\quad + p_K \ell_K - \sum_{j\ne K}p_K(C-\ell_j),  
        \\
        =& \sum_{j=1}^K p_j \ell_j + \sum_{j\ne K} (p_j-p_K)(C-\ell_j),  
        \\
        \leq & \sum_{j=1}^K p_j \ell_j. \qedhere
        % ,\quad (\text{Using assumptions }(a)\text{ and }(c)).
    \end{align*}
\end{proof}

\universal*
\begin{proof}
    %\textbf{(if)}
    % direct
    % first part
    % Let $\bp$ be an arbitrary probability vector. 
    If $\ell$ is conservative, then we can use \Cref{lemma:earth-moving} with the sorted version of $\bp$ and constants $\{\ell(\widehat{\by},\be_j)\}_{j=1}^K$ to give
    \begin{align}
        \widehat{R}(\widehat{\by};\cS) = \WS\sum_{j=1}^K p_j\ell(\widehat{\by},\be_j) \geq C\WS (1-\|\bp\|_\infty).\label{eq:lb}
    \end{align}
    % second part
    The lower bound in \Cref{eq:lb} can be achieved by choosing $k=\argmax(\bp),\ \widehat{\by}=\be_k$. 
    Indeed, by $(a)$, $(b)$ and $(c)$ from the definition of the conservative loss, we have $\ell(\be_k,\be_k)=0$ and $\ell(\be_k,\be_j)=C,\ j\ne k$, giving 
    \begin{align*}
        \widehat{R}^*(\cS)
        &= \widehat{R}(\be_{k};\cS),
        \\
        &=\WS\sum_{j=1}^K p_j\ell(\be_k,\be_j),
        \\
        &=C\WS\sum_{j\ne k} p_j,
        \\
        &=C\WS (1-\|\bp\|_\infty).
    \end{align*}
    
    %\textbf{(only if)}
    % contrapositive
    % \forall p, A(p) implies B and C
    % equiv
    % not B or not C implies \exist p s.t. not A(p)
    % Fix $\widehat{\by}\in\Delta^{K-1}$. 
    % Fix $\widehat{\by}\in\simplex$ and
    Now we show that \Cref{eq:misclassification_impurity} implies (a) in \Cref{defn:conservative}.
    Assume $(a)$ is false, then there exists $\widehat{y}$ such that
    $\sum_{j=1}^K\ell(\widehat{\by},\be_j) < C(K-1)$. 
    Consider an $\cS$ with its corresponding $\bp$ being the uniform distribution, then, 
    \begin{align*}
        &\quad \widehat{R}^{*}(\cS) \\
        &\le \widehat{R}(\widehat{\by};\cS) 
            = \frac{1}{|\cD|}\sum_{(\bx,\by)\in\cS} \ell(\widehat{\by},\by)\\
        &= \WS \sum_{j=1}^K p_j\ell(\widehat{\by},\be_j)
            = \WS\frac{1}{K}\sum_{j=1}^K \ell(\widehat{\by},\be_j) \\
        &< \WS\frac{1}{K}C(K-1)
            =  C\WS (1-1/K) \\
        &= C \WS (1-\|\bp\|_\infty).
    \end{align*}
    That is, $\widehat{R}^{*}(\cS) < C \WS (1-\|\bp\|_\infty)$, a contradiction.
    Therefore, \Cref{eq:misclassification_impurity} implies (a).
  %   However, \Cref{eq:misclassification_impurity} states that
		% $\min_{\widehat{\by} \in \simplex}\,\widehat{R}(\widehat{\by};\cS)=C\WS(1-\|\bp\|_\infty)$. Having shown the contrapositive, we can infer that if \Cref{eq:misclassification_impurity} holds, then each of $(a)$, $(b)$ and $(c)$ also hold. 
\end{proof}

\section{Proof of \Cref{cor:robust_losses}}
\label{app:cor_robust_losses}
\rl*
\begin{proof}
To show this, we need to show that the conditions of \Cref{thm:universal} are satisfied for each loss function. 
Fix 
$\widehat{\by}\in\simplex$ 
and 
$k\in\{1,\ldots,K\}$. We first check condition $(a)$: 
\begin{align*}
    \ell( & \widehat{\by}  ,\be_k) =
    \\
    &\begin{cases}
        \|\widehat{\by}-\be_k\|_1 = 2-2\widehat{y}_k \leq 2 & \mathrm{MAE},
        \\
        \I(\widehat{\by}\ne \be_k)  \leq 1 & 01,
        \\
        \frac{1}{q}(1-(\be_k^\T\widehat{\by})^q) = \frac{1}{q}(1-\widehat{y}_k^q) \leq \frac{1}{q} & \mathrm{GCE},
        \\
        \|\widehat{\by}-\be_k\|_\infty = \max(1-\widehat{y}_k,\,\underset{j\ne k}{\max}\,\widehat{y}_j)\leq 1 & \mathrm{INF}.
    \end{cases}
\end{align*}
Then for condition $(b)$:
\begin{align*}
    \sum_{j=1}^K & \ell(\widehat{\by},\be_j) = 
        \\
        &\begin{cases}
        \sum_{j=1}^K (2-2\widehat{y}_j) = 2(K-1) & \mathrm{MAE},
        \\
        \sum_{j=1}^K \I(\widehat{\by}\ne \be_j) \geq K-1 & 01,
        \\
        \sum_{j=1}^K \frac{1}{q}(1-\widehat{y}_j^q) 
        \geq 
        % \frac{1}{q}\sum_{j=1}^K(1-\widehat{y}_j)
        % =
        \frac{1}{q}(K-1)   & \mathrm{GCE},
        \\
        \sum_{j=1}^K \max(1-\widehat{y}_j,\,\underset{i\ne j}{\max}\,\widehat{y}_i)
        \geq 
        % \sum_{j=1}^K (1-\widehat{y}_j) 
        % = 
        K-1 & \mathrm{INF}.
    \end{cases}
\end{align*}
Finally condition $(c)$. Note that for $\widehat{\by}=\be_k$, we have $\widehat{y}_k=1$, so,
\[
    \ell(\be_k,\be_k) = 
        \begin{cases}
        (2-2) = 0 & \mathrm{MAE},
        \\
        \I(1=0) = 0 & 01,
        \\
        \frac{1}{q}(1-1^q) =0   & \mathrm{GCE},
        \\
        \max(1-1,\,0) = 0 & \mathrm{INF}.
    \end{cases} \qedhere
\]
\end{proof}

\section{Proof of \Cref{cor:one-hot-predictions}}
\label{app:one-hot-predictions}
\ohp*
\begin{proof}
    If $\ell$ is conservative, then the if-part of the proof of \Cref{thm:universal} shows that the 
    minimum partial impurity is achieved by $\widehat{\by} = \be_{\argmax(\bp)}$, with
    \[
        % \underset{\widehat{\by}\in\simplex}{\min}\,
        \widehat{R}(\be_{\argmax(\bp)};\cS) = C(1-\|\bp\|_\infty).
    \]
    % for $\widehat{\by}=\be_{\argmax(\bp)}$
    As seen in the proof of \Cref{thm:loss_impurity_equivalence} $(a)$, the partial empirical risk with MSE loss takes on its minimum value of the Gini impurity for $\widehat{\by}=\bp$. 
    % since
    % \begin{align*}
    %     \widehat{R}_{\mathrm{MSE}}(\bp;\cS) &= \frac{1}{|\cD|}\sum_{(\bx,\by)\in\cS} \|\bp-\by\|_2^2,
    %     \\
    %     &= \frac{1}{|\cD|} \sum_{(\bx,\by)\in\cS} \|\bp\|_2^2+1-2\bp^\T\by,
    %     \\
    %     &= \WS\|\bp\|_2^2 +\WS -  2\WS\|\bp\|_2^2,
    %     \\
    %     &= \WS(1-\|\bp\|_2^2).
    % \end{align*}
    Similarly for the CE loss, as seen in the proof of \Cref{thm:loss_impurity_equivalence} $(b)$, the partial empirical risk takes on its minimum value of the entropy impurity for $\widehat{\by}=\bp$.
    % , since
    % \begin{align*}
    %     \widehat{R}_{\mathrm{CE}}(\bp;\cS) &= -\frac{1}{|\cD|}\sum_{(\bx,\by)\in\cS} \by^\T\log \bp,
    %     \\
    %     &= -\WS\sum_{j=1}^K p_j\log p_j.
    % \end{align*}
\end{proof}

\section{Proof of \Cref{thm:hoeffding}}
\label{app:hoeffding}

\hoeffding*
\begin{proof}

    The probability that the majority class in $\widetilde{\cS}$ remains the same as in $\cS$ is
    \begin{align*}
        &\; \Pm
        (
        \widetilde{p}_{k^*}\geq\widetilde{p}_{k}
        \text{for all $k\ne k^*$}
        \gvn \cS
        ) 
        \\
        =&\; 1-\Pm
        (
        % \bigcup_{k\ne k^*} 
        % \{
        \widetilde{p}_{k^*} < \widetilde{p}_{k}
        \text{for at least one $k\ne k^*$}
        \gvn \cS
        ). 
    \end{align*}
    By Boole's inequality (i.e., the union bound),
    \begin{align}
        &\; 1-\Pm
        (
        \widetilde{p}_{k^*} < \widetilde{p}_{k}
        \text{for at least one $k\ne k^*$}
        \gvn \cS
        ) \nonumber
        \\
        \geq&\; 1 - \sum_{k\ne k^*} \Pm\left( \widetilde{p}_{k^*} < \widetilde{p}_{k} 
        \gvn \cS
        \right) \nonumber
        \\
        \geq&\; 1 - \sum_{k\ne k^*} \Pm\left( \widetilde{p}_{k} - \widetilde{p}_{k^*} \geq 0  
        \gvn 
        % p_{k^*}>p_{k}
        \cS
        \right).\label{eq:booles}
    \end{align}

    We shall use Hoeffding's inequality to bound each 
    $\Pm\left( \widetilde{p}_{k} - \widetilde{p}_{k^*} \geq 0  
        \gvn 
        % p_{k^*}>p_{k}
        \cS
        \right)$ for each $k \neq k^{*}$, because
    $\widetilde{p}_{k} - \widetilde{p}_{k^*}$ is the average of the 
    independent random variables
    $\{\widetilde{y}_{ik^*} - \widetilde{y}_{ik}\}_{i=1}^n$
    which take values in $[-1, 1]$.
        
    Since the noise is uniform with rate $\eta$, we have
    \begin{align*}
        \widetilde{p}_{k}^{\eta} 
        &= \Em[\widetilde{p}_{k}] 
        = (1-\eta) p_{k} + \frac{\eta}{K-1} \sum_{j \neq k} p_{j} \\
        &= \left(1 - \frac{K\eta}{K-1}\right)p_{j} + \frac{\eta}{K-1}.
    \end{align*}
    In addition, since $\eta < (K-1)/K$, we have for any $k \neq k^{*}$,
    \begin{align*}
        t_{k} := \widetilde{p}_{k^{*}}^{\eta} - \widetilde{p}_{k}^{\eta}
        = \left(1 - \frac{K\eta}{K-1}\right)(p_{k^{*}} - p_{k}) 
        > 0.
    \end{align*}

    Applying Hoeffding's inequality, we have 
    \begin{align}
        &\; \Pm\left( \widetilde{p}_{k} - \widetilde{p}_{k^*} \geq 0 \gvn \cS\right) \\
        =&\; \Pm\left( \widetilde{p}_{k} - \widetilde{p}_{k^*} \geq t_{k} + (\widetilde{p}_{k}^{\eta} - \widetilde{p}_{k^*}^{\eta}) 
        \gvn \cS \right) \nonumber \\
        \leq&\; \exp\left( -nt_{k}^2/2 \right) \nonumber \\
        \leq&\; \exp\left( -n \gamma^{2}/2 \right), \label{eq:lower}
    \end{align}
    where $\gamma = \underset{k\ne k^*}{\min}\,(p_{k^*} - p_{k})$ is the margin.

    Plugging the bound in \Cref{eq:lower} into \Cref{eq:booles}, we have
    \begin{align*}
        \Pm(\widetilde{p}_{k^*}\geq \widetilde{p}_{k} \text{ for all $k \neq k^{*}$}
        \gvn \cS) 
        &\geq 1-(K-1) e^{-n\gamma^2/2}.\qedhere
    \end{align*}

\end{proof}

\section{Proof of \Cref{thm:early_stopping}}
\label{app:early_stopping}
The proof of \Cref{thm:early_stopping} relies on the result in \Cref{lemma:max-norm-equality}.

\begin{restatable}[]{lemma}{maxnormequality}
    \label{lemma:max-norm-equality} 
    Let $\bx=(x_1,\ldots,x_d)^\T,\by=(y_1,\ldots,y_d)^\T\in \R_{\geq 0}^d$ be $d$-dimensional vectors of nonnegative real numbers and 
    $\sM_{\bz}=\{j:z_j=\|\bz\|_\infty\}$ 
    be the set of indices corresponding to maximal elements of $\bz$, for any $\bz\in \R_{\geq 0}^d$. Then, 
    \[
        \|\bx+\by\|_\infty = \|\bx\|_\infty + \|\by\|_\infty
    \]
    if and only if 
    \[
        \sM_{\bx+\by} = \sM_{\bx}\cap \sM_{\by}.
    \]
\end{restatable}

\begin{proof}
    \textbf{(if)}
    Assume $\sM_{\bx+\by}=\sM_{\bx} \cap \sM_{\by}$, and let $i\in \sM_{\bx+\by}$. Then, $i\in \sM_{\bx}$, $i\in \sM_{\by}$, and
    \begin{align*}
        \|\bx+\by\|_\infty &= x_i+y_i,
        \\
        \|\bx\|_\infty &= x_i,
        \\
        \|\by\|_\infty &= y_i.
    \end{align*}
    Thus, $\|\bx+\by\|_\infty = \|\bx\|_\infty + \|\by\|_\infty$.
    
    \textbf{(only if)}
    Assume $\|\bx+\by\|_\infty = \|\bx\|_\infty + \|\by\|_\infty$, and let 
    $i \in \mathscr{M}_{\bx}$, $j \in \mathscr{M}_{\by}$, $k \in \mathscr{M}_{\bx+\by}$. 
    % \begin{align*}
    %     j 
    %     % &= \underset{i=1,\ldots,d}{\argmax}\,|x_i|,
    %     &\in \mathscr{M}_{\bx},
    %     \\
    %     k 
    %     % &= \underset{i=1,\ldots,d}{\argmax}\,|y_i|,
    %     &\in \mathscr{M}_{\by},
    %     \\
    %     l 
    %     % &= \underset{i=1,\ldots,d}{\argmax}\,|x_i+y_i|.
    %     &\in \mathscr{M}_{\bx+\by}.
    % \end{align*}
     Then, by assumption, we have
     \begin{align*}
         x_k+y_k = x_i+y_j,
     \end{align*}
     if and only if 
     \begin{align}
         x_k-x_i = y_j - y_k. \label{eq:compare_maxs}
     \end{align}
     The LHS of \Cref{eq:compare_maxs} is at most zero, and the RHS of \Cref{eq:compare_maxs} is at least zero. We achieve equality if and only if both sides are zero, which happens when 
     $k\in\sM_{\bx}\cap \sM_{\by}$. 
     Since $k$ is an arbitrary element of 
     $\sM_{\bx+\by}$, 
     we must have 
     $\sM_{\bx+\by}\subseteq \sM_{\bx}\cap \sM_{\by}$. 
     If we let 
     $l\in \sM_{\bx}\cap \sM_{\by}$, 
     then by analogous reasoning we must have 
     $l\in\sM_{\bx+\by}$. 
     Therefore, 
     \[
        \sM_{\bx+\by}=\sM_{\bx}\cap\sM_{\by}. \qedhere
     \]
\end{proof}

\es*

\begin{proof}
    {\bf (a)}
    Consider a conservative loss $\ell$ with associated positive constant $C$,
    a node containing a set of examples $\cS$,
    and an arbitrary split $(f, t)$.
    Let $\bp_{\cS}$ be the vector of proportions of data in $\cS$ from each class,
    and $\bm{n}_{\cS}$ be the vector containing the number of examples in $\cS$ from each class. 
    % Then, the risk reduction is 
    % \begin{align*}
    %     \mathrm{RR}(f,t;\cS) =\ & C\WS (1-\|\bp_{\cS}\|_\infty)
    %     \\& - C W_{\cS_{f\leq t}} (1-\|\bp_{\cS_{f\leq t}}\|_\infty)
    %     \\& - C W_{\cS_{f> t}} (1-\|\bp_{\cS_{f > t}}\|_\infty).
    % \end{align*}
    Then, we have $\RR(f,t;\cS) \leq 0 $ if and only if
    \begin{align*}
        C\WS (1-\|\bp_{\cS}\|_\infty)  
        \leq\ & C W_{\cS_{f\leq t}} (1-\|\bp_{\cS_{f\leq t}}\|_\infty)
        \\ & + C W_{\cS_{f> t}} (1-\|\bp_{\cS_{f > t}}\|_\infty),
    \end{align*}
    which is equivalent to
    \begin{align*}
        |\cS| (1-\|\bp_{\cS}\|_\infty) \leq &\ 
         |\cS_{f\leq t}| (1-\|\bp_{\cS_{f\leq t}}\|_\infty)
        \\ &+|\cS_{f> t}| (1-\|\bp_{\cS_{f > t}}\|_\infty).
    \end{align*}
    Using $|\cS_{f\leq t}| + |\cS_{f > t}| = |\cS|$ gives:
    \begin{align}
         % |\cS| \|\bp_{\cS}\|_\infty &\geq 
         %  |\cS_{f\leq t}| \|\bp_{\cS_{f\leq t}}\|_\infty
         % + |\cS_{f> t}|\|\bp_{\cS_{f > t}}\|_\infty,\nonumber
         % \\
         \|\bm{n}_{\cS}\|_\infty &\geq
           \|\bm{n}_{\cS_{f\leq t}}\|_\infty
         + \|\bm{n}_{\cS_{f > t}}\|_\infty.\label{eq:mi_es}
    \end{align}
    However, it is always true that $\|\bm{n}_{\cS}\|_\infty \leq \|\bm{n}_{\cS_{f\leq t}}\|_\infty + \|\bm{n}_{\cS_{f > t}}\|_\infty$ by the triangle inequality, hence we stop splitting when equality in \eqref{eq:mi_es} is achieved. 
    By \Cref{lemma:max-norm-equality}, this happens only when the majority classes at the parent node
    are also the majority classes at both child nodes for the split $(f, t)$.
    % We can apply \Cref{lemma:max-norm-equality} with vectors $\bn_{\cS_{f\leq t}},\bn_{\cS_{f\leq t}}$ and sets $\cM_{\cS_{f\leq t}},\cM_{\cS_{f > t}}$ to say 
    % \[
    %     \|\bm{n}_{\cS}\|_\infty =
    %        \|\bm{n}_{\cS_{f\leq t}}\|_\infty
    %      + \|\bm{n}_{\cS_{f > t}}\|_\infty 
    % \]
    % if and only if
    % \[
    %      \cM_{\cS} = \cM_{\cS_{f\leq t}} \cap \cM_{\cS_{f > t}}. 
    % \]

    \medskip{\bf (b)}
    For the MSE, it was shown in \cite{breiman1984classification} that $\RR_{\mathrm{MSE}}(f,t;\cS)\leq 0$ only when $\bp_{\cS}=\bp_{\cS_{f\leq t}}=\bp_{\cS_{f > t}}$ since $\widehat{R}_{\mathrm{MSE}}^*$, 
    when viewed as a function of $\bp_{\cS}$, 
    is strictly concave.
    
    Since $\widehat{R}^*$ is also strictly concave for both the CE and GCE $(q\in(0,1))$ loss when viewed as a function of $\bp_{\cS}$, the same result holds as for MSE.
    % By the same argument as for MSE, we have $\RR(f,t;\cS)\leq 0$ only when each of the parent and two child nodes have the same label distribution for the CE and GCE $(q\in(0,1))$ loss, since in both cases the minimum partial empirical risk is strictly concave.
    % For each of the MSE, CE and GCE $(q\in(0,1))$, the minimum partial empirical risk $\widehat{R}^*(\cS)$ when viewed as a function of $\bp$ is (strictly) concave. 
    % Indeed, the hessian matrix of $\widehat{R}_{\mathrm{MSE}}^*(\cS)$ is $-2\mathbf{I}$, and the hessian matrix of $\widehat{R}_{\mathrm{CE}}^*(\cS)$ is $-\mathrm{diag}((1/ p_k)_{k=1}^K)$, both of which are negative definite. 
    % For GCE $(q\in(0,1))$, we can use Minkowski's inequality to show $\|\bp\|_{1/(1-q)}$ is convex, indeed, for any $\alpha\in[0,1]$ and $\bp_L,\bp_R\in\cP$,
    % \begin{align*}
    %     & \|\alpha\bp_L + (1-\alpha)\bp_R\|_{1/(q-1)}
    %     \\
    %     & \leq \alpha\|\bp_L\|_{1/(q-1)} 
    %      + (1-\alpha)\|\bp_R\|_{1/(q-1)},
    % \end{align*}
    % with equality if and only if $\bp_L=\bp_R$. Hence, $\widehat{R}_{\mathrm{GCE}}^*(\cS)$ is also (strictly) concave.
    % By concavity, for the three aforementioned loss functions, $\RR(f,t;\cS)\geq 0$, with equality if and only if the two child nodes have the same label distributions. 
\end{proof}

\section{Proof of \Cref{lem:distloss}}
\distloss*
\begin{proof}
    % The cdf of the $\ber(p)$ distribution is 
    % \begin{align*}
    %     F(x)=(1-p)\I(x\in[0,1))+\I(x\geq 1).
    % \end{align*}
    %For $p=0$ the cdf simplifies to 
    The CDF of the $\ber(0)$ distribution is
    \begin{align*}
        F(x)=\I(x\geq 0)=(1-\mathrm{sign}(-x))/2.
    \end{align*}
    The loss function is thus 
    \begin{align*}
        \ell(z)=F(-z)=(1-\mathrm{sign}(z))/2,
    \end{align*}
    which is the 01 loss.

    \medskip
    The CDF of the Logistic distribution with location 0 and scale 1 is 
    \begin{align*}
        F(x)=1/(1+\exp(-x)).
    \end{align*} 
    The corresponding loss function is 
    \begin{align*}
        \ell(z)=F(-z)=1/(1+\exp(x)),
    \end{align*} 
    which is the sigmoid loss.

    \medskip
    The CDF of the $\mathsf{U}(-1,1)$ distribution is 
    \begin{align*}
        F(x) 
        &=\frac{x+1}{2}\I(x\in [-1,1])+\I(x>1) \\
        &=\max\{0,\min\{1,(1+z)/2\}\}.
    \end{align*} 
    The corresponding loss function is 
    \begin{align*}
        \ell(z)=F(-z)=\max\{0,\min\{1,(1-z)/2\}\},
    \end{align*}
    which is the ramp loss.
\end{proof}

\section{Proof of \Cref{lem:neloss}}
\neloss*
\begin{proof}
    Let $X$ be a random variable having an exponential distribution with scale $\beta$. Then the CDF of $X$ is
    $F(x)=\max\{0,1-\exp(-x/\beta)\}$.
    The CDF of its shifted negative $Z=\mu-X$ is
    \begin{align}
        F(z) &= 1-\max\{0,1-\exp(-(\mu-z)/\beta)\}\nonumber
        \\
        &=\min\{1,\exp((z-\mu)/\beta)\}.\label{eq:ycdf}
    \end{align}
    Using $\beta=1$ and plugging $z \gets -y\hat{y}$ into \Cref{eq:ycdf} gives the result. 
\end{proof}

\section{Proof of \Cref{thm:negexp}}
\label{app:NE-loss}
\negexp*

\begin{figure*}[t]
    \centering
    \includegraphics[width=14cm]{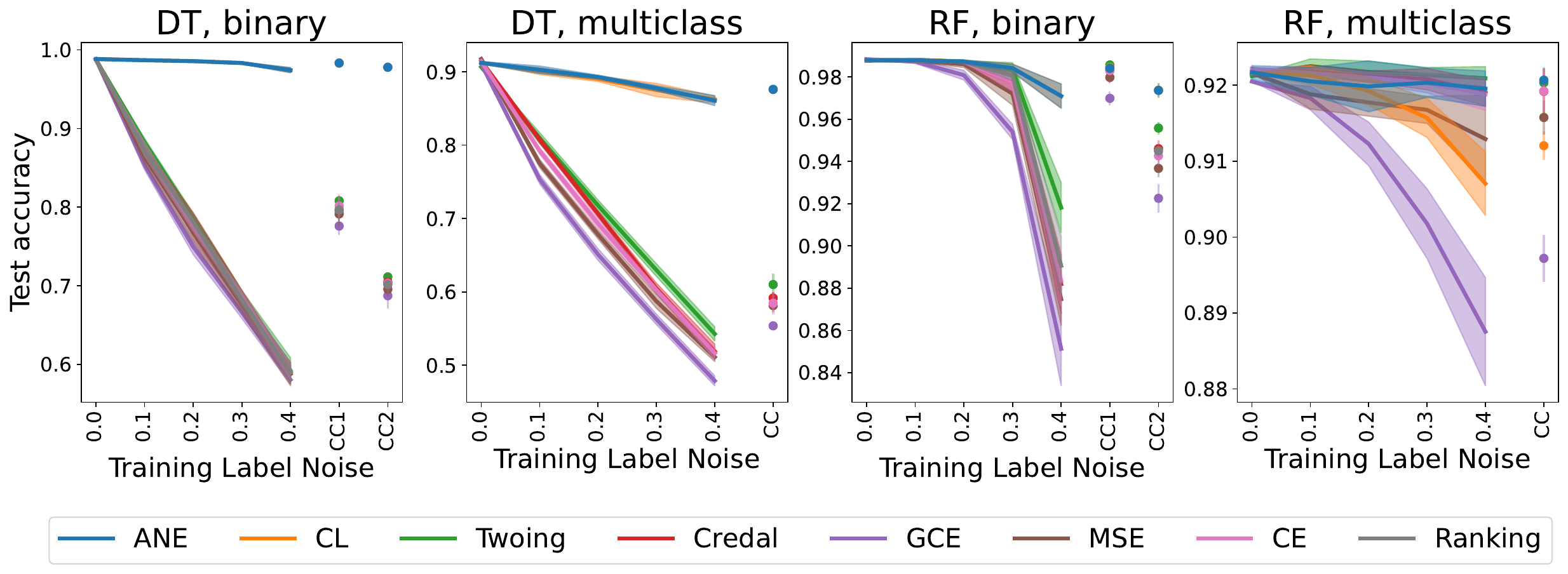} 
    \caption{Mean test accuracy with 2x standard deviation bands for decision tree and random forest with different splitting criteria. 
    Training labels corrupted using uniform noise $\eta\in\{0.0,0.1,0.2,0.3,0.4\}$ and class conditional noise CC1, CC2 for binary classification and CC for multiclass classification. 
    }
    \label{fig:pre_trained}
\end{figure*}

\begin{proof}
    Simplifying the partial empirical risk by grouping together the positive and negative samples gives
    \begin{align*}
    &\; \widehat{R}(\widehat{y};\cS) \\
        =&\; \sum_{(\bx,y)\in\cS} \min\{1, e^{-\widehat{y}y-\mu)}\}/|\cD| \\
        =&\; \WS \left( 
                p_{-} \min\{1, e^{\widehat{y}-\mu}\}
                +
                p_{+} \min\{1,e^{-\widehat{y}-\mu}\}
                \right) \\
        =&\;
        \begin{cases} 
            \WS \left(p_{-} e^{\widehat{y}-\mu} + p_{+} \right), 
                & \widehat{y}\leq -\mu, \\
            \WS \left(p_{-} e^{\widehat{y} - \mu} + p_{+} e^{-\widehat{y}-\mu}\right), 
                & -\mu<\widehat{y}<\mu, \\
            \WS \left(p_{-} + p_{+} e^{-\widehat{y}-\mu}\right), 
                & \widehat{y}\geq \mu.
            \end{cases} 
    \end{align*}
    We consider the above three cases separately.
    \begin{itemize}
        \item When $\widehat{y} \leq -\mu$, the infimum is $p_{+}$, achieved when $\widehat{y} = -\infty$.
        \item When $-\mu<\widehat{y}<\mu$, using the fact that $a + b \ge 2\sqrt{ab}$ for
            $a, b \ge 0$ with equality when $a = b$, we have 
            \begin{align*}
                p_{-} e^{\widehat{y} - \mu} + p_{+} e^{-\widehat{y}-\mu}
                %\ge 2 \sqrt{p_{-} e^{\widehat{y} - \mu} p_{+} e^{-\widehat{y}-\mu}}
                \ge 2 e^{-\mu} \sqrt{p_{-} p_{+}}.
            \end{align*}
            The infimum is obtained when 
            $p_{-} e^{\widehat{y} - \mu} = p_{+} e^{-\widehat{y}-\mu}$.
        \item
            When $\widehat{y} \geq \mu$, the infimum is $p_{-}$, achieved when
            $\widehat{y} = +\infty$.
    \end{itemize}

    Therefore we can write
    \begin{align*}
        \widehat{R}^*(\cS) &= \min\{p,1-p,2e^{-\mu}\sqrt{p(1-p)}\}
        \\
        &= \min\left\{1-\|\bp\|_\infty,\, \lambda\sqrt{(1-\|\bp\|_2^2)/2} \right\}. \qedhere
    \end{align*}
\end{proof}

\section{Class Conditional Noise}
\label{app:class-conditional-noise}

% Method for corrupting labels follows principle: data points from different classes but with similar feature values are more difficult to distinguish. 

The method used for corrupting labels with class conditional noise in multiclass classification problems is based on the Mahalanobis distance between classes.
Let $\overline{\bx}_1$ and $\overline{\bx}_2$ be the sample mean vectors of examples in classes 1 and 2, respectively, 
and $\mathbf{S}$ be the estimated common covariance matrix computed using feature vectors for the two classes. 
Then, the Mahalanobis distance between class 1 and 2 is defined to be 
$\sqrt{(\overline{\bx}_1-\overline{\bx}_2)^\T \mathbf{S}^{-1} (\overline{\bx}_1-\overline{\bx}_2)}$.
% In summary, two classes that are similar (in terms of Mahalanobis distance) to each other will have 

% samples from classes that are similar to many other classes will have their labels corrupted often, with higher chances of being mislabeled as belonging to similar classes. 

Transition probabilities on the off diagonals of $(\eta_{ij})_{i,j=1}^K$ are larger for pairs of classes with high similarity (in terms of Mahalanobis distance), and transition probabilities on the diagonals are smaller for those classes whose total similarity to every other class is high.
Pseudocode for the algorithm used to generate noise rates is given in \Cref{alg:ccn}. 
% A brief summary of how the algorithm works and the nature of the transition probabilities between classes:
% \begin{itemize}
%     \item The Mahalanobis distance between each class is computed and turned into a similarity measure by taking the reciprocal.
%     \item Classes are ranked by total similarity to every other class, with the class most similar to the others receiving a score of 0.5 and the least similar class receiving a score of 0.9. The other classes scores are interpolated linearly between 0.5 and 0.9 based on their total similarities to every other class.
%     \item Theses scores become the probabilities of an example from each class receiving the correct label (diagonal of transition matrix).
%     \item The off diagonal entries of the transition matrix are the similarity measures between the respective classes scaled to ensure that the row sums of the transition matrix are one. 
% \end{itemize}

% With these noise rates, classes that are more similar to each other are more likely to have their labels switched, and classes that are more similar to every other class are more likely to be corrupted in the first place.

\begin{algorithm}[h]
\caption{Class Conditional Noise}
\label{alg:ccn}
\textbf{Input:} Data $\{(\bx_i,\by_i)\}_{i=1}^n$.\\ 
\textbf{Output:} Transition probabilities $(\eta_{ij})_{i,j=1}^K$.

% \begin{algorithmic}[1]
% \STATE Initialize similarity matrix $(s_{ij})_{i,j=1}^K$.\\
% \FOR{$i=1,\ldots,K$}
%     \FOR{$j=1,\ldots,K$}
%         \IF{$i\ne j$}
%             \STATE $d\gets$ Mahalanobis distance between class $i$ and class $j$\\
%             \STATE $s_{ij}\gets 1/d$\\
%         \ENDIF
%     \ENDFOR
% \ENDFOR
% \STATE $(weights)_{i=1}^K\gets $ row sums of $(1/s_{ij})_{i,j=1}^K$\\
% \STATE $(diags)_{i=1}^K\gets$ scaled $(weights)_{i=1}^K$ between 0.5 and 0.9\\
% \FOR{$i=1,\ldots,K$}
%     \STATE scale row $i$ of the similarity matrix $(s_{ij})_{i,j=1}^K$ so that the row sum is $1-diag_i$\\
% \ENDFOR
% \STATE $(\eta_{ij})_{i,j=1}^K\gets (s_{ij})_{i,j=1}^K+\mathrm{diag}((diags)_{i=1}^K)$\\
% \STATE \textbf{Return}: Transition probabilities $(\eta_{ij})_{i,j=1}^K$
% \end{algorithmic}
\begin{algorithmic}[1]
\STATE Compute $s_{ij}$, the \emph{inverse} Mahalanobis distance between class $i$ and class $j$, for all $i \neq j$.
\STATE Compute $s_{ii} = \sum_{j \neq i} 1/s_{ij}$ for each class $i$.
\STATE Scale $s_{ii}$ to range from 0.5 to 0.9.
\STATE For all $i, j$, set $\eta_{ij}$ to $\begin{cases}
    s_{ij}/(1 - s_{ii}),  & i \neq j,\\
    s_{ii}, & i = j.
    \end{cases}$.
\STATE \textbf{Return}: Transition probabilities $(\eta_{ij})_{i,j=1}^K$
\end{algorithmic}
\end{algorithm}

\section{Pre-trained Image Embeddings CIFAR-10}
\label{app:cifar10-pretrained}

We performed additional experiments on CIFAR-10 using pre-trained image embeddings from VGG11\footnote{\url{https://github.com/huyvnphan/PyTorch_CIFAR10}}. 
The input to the last classifier module in the network was taken to be the new feature vectors.
Other processing steps are the same as in the other experiments.

From 
\Cref{fig:pre_trained}
we see that predictive performance of all tree methods generally improve compared to using the raw pixel values as in \Cref{fig:dt-bin,fig:dt-mc,fig:rf-bin,fig:rf-mc}. 
ANE still shows strong performance, 
generally performing the best across all noise settings,
and particularly with large amounts of noise.
% Conservative loss performs similarly to ANE in most of these experiments,  

\section{Detailed Experiment Results}
\label{app:results}

\Cref{tab:dt-bin,tab:dt-mc,tab:rf-bin,tab:rf-mc} present the average accuracies and 2 standard deviations 
used in \Cref{fig:dt-bin,fig:dt-mc,fig:rf-bin,fig:rf-mc}. 
Best performers (and equivalent) are highlighted for each dataset and noise setting.

% Methods giving the largest mean accuracy and those with overlapping 97.5\% confidence interval are highlighted for each dataset and noise setting.

\begin{table*}[t] 
	\scriptsize 
	\centering 
	\begin{tabular}{clccccccc} 
&& \multicolumn{5}{c}{Uniform Noise} & \multicolumn{2}{c}{Class Conditional} \\ \cline{3-9} 
Dataset & Split Crit. & $0.0$ & $0.1$ & $0.2$ & $0.3$ & $0.4$ & $(0.1,0.3)$ & $(0.2,0.4)$  \\ \hline 
\multirow{8}{*}{MNIST}
 & ANE    & $\mathbf{94.00 \pm 0.48}$ & $\mathbf{89.92 \pm 0.84}$ & $\mathbf{87.89 \pm 1.37}$ & $\mathbf{84.13 \pm 3.35}$ & $\mathbf{82.03 \pm 1.39}$ & $\mathbf{82.20 \pm 1.35}$ & $\mathbf{77.80 \pm 1.65}$ \\ 
 & CL     & $88.14 \pm 0.08$ & $86.94 \pm 2.88$ & $\mathbf{86.07 \pm 3.04}$ & $\mathbf{84.76 \pm 2.42}$ & $\mathbf{81.56 \pm 2.06}$ & $\mathbf{82.41 \pm 1.19}$ & $\mathbf{79.03 \pm 2.49}$ \\ 
 & Twoing & $90.34 \pm 0.48$ & $82.99 \pm 1.00$ & $75.31 \pm 1.52$ & $66.66 \pm 0.92$ & $57.98 \pm 1.28$ & $75.09 \pm 0.49$ & $67.11 \pm 0.56$ \\ 
 & Credal & $\mathbf{94.39 \pm 0.31}$ & $84.23 \pm 0.24$ & $74.86 \pm 0.51$ & $66.52 \pm 1.30$ & $58.18 \pm 0.54$ & $75.56 \pm 1.80$ & $66.70 \pm 1.35$ \\ 
 & GCE    & $93.17 \pm 0.06$ & $83.04 \pm 1.16$ & $74.22 \pm 0.66$ & $65.88 \pm 0.44$ & $58.13 \pm 0.48$ & $74.53 \pm 1.08$ & $66.20 \pm 1.65$ \\ 
 & MSE    & $94.08 \pm 0.13$ & $83.59 \pm 0.89$ & $74.60 \pm 0.50$ & $66.14 \pm 1.05$ & $57.70 \pm 1.13$ & $75.54 \pm 0.89$ & $66.56 \pm 1.06$ \\ 
 & CE     & $\mathbf{94.47 \pm 0.21}$ & $84.10 \pm 0.64$ & $74.85 \pm 0.90$ & $66.51 \pm 0.93$ & $58.19 \pm 0.65$ & $75.61 \pm 0.84$ & $66.86 \pm 0.94$ \\ 
 & Ranking & $94.02 \pm 0.11$ & $83.75 \pm 0.69$ & $74.93 \pm 0.56$ & $66.12 \pm 1.09$ & $57.98 \pm 0.91$ & $75.50 \pm 0.51$ & $66.44 \pm 1.50$ \\ 
\hline
\multirow{8}{*}{20News}
 & ANE    & $\mathbf{75.31 \pm 1.10}$ & $\mathbf{74.30 \pm 1.26}$ & $72.23 \pm 0.74$ & $\mathbf{69.06 \pm 1.77}$ & $\mathbf{59.75 \pm 8.84}$ & $\mathbf{69.12 \pm 1.36}$ & $\mathbf{62.38 \pm 2.96}$ \\ 
 & CL     & $\mathbf{75.76 \pm 0.20}$ & $\mathbf{74.68 \pm 0.44}$ & $\mathbf{73.27 \pm 0.48}$ & $\mathbf{69.88 \pm 1.33}$ & $\mathbf{64.54 \pm 2.06}$ & $\mathbf{69.50 \pm 1.57}$ & $\mathbf{63.29 \pm 1.81}$ \\ 
 & Twoing & $72.75 \pm 0.97$ & $68.42 \pm 1.14$ & $64.07 \pm 0.66$ & $60.20 \pm 1.22$ & $54.75 \pm 1.29$ & $62.80 \pm 1.47$ & $58.25 \pm 1.30$ \\ 
 & Credal & $74.92 \pm 0.24$ & $69.75 \pm 0.66$ & $64.40 \pm 1.34$ & $59.68 \pm 0.96$ & $55.69 \pm 1.62$ & $63.59 \pm 0.85$ & $58.44 \pm 1.06$ \\ 
 & GCE    & $72.41 \pm 1.05$ & $67.63 \pm 1.51$ & $62.94 \pm 1.25$ & $58.44 \pm 1.16$ & $54.65 \pm 2.08$ & $61.85 \pm 0.83$ & $57.50 \pm 1.59$ \\ 
 & MSE    & $74.94 \pm 0.40$ & $69.05 \pm 1.50$ & $64.52 \pm 0.57$ & $59.68 \pm 2.79$ & $54.94 \pm 1.60$ & $63.33 \pm 1.08$ & $57.81 \pm 1.03$ \\ 
 & CE     & $\mathbf{75.84 \pm 0.40}$ & $70.02 \pm 0.40$ & $65.56 \pm 1.02$ & $59.66 \pm 1.11$ & $54.82 \pm 0.41$ & $63.82 \pm 1.24$ & $58.61 \pm 1.38$ \\ 
 & Ranking & $74.51 \pm 0.71$ & $69.37 \pm 0.63$ & $64.56 \pm 0.71$ & $59.61 \pm 1.10$ & $54.96 \pm 1.78$ & $63.89 \pm 1.20$ & $58.68 \pm 1.73$ \\ 
\hline
\multirow{8}{*}{UNSW-NB15}
 & ANE    & $\mathbf{85.90 \pm 1.49}$ & $\mathbf{86.28 \pm 0.21}$ & $\mathbf{81.47 \pm 0.23}$ & $\mathbf{80.89 \pm 0.31}$ & $\mathbf{80.45 \pm 2.40}$ & $\mathbf{88.12 \pm 0.68}$ & $\mathbf{83.05 \pm 4.23}$ \\ 
 & CL     & $80.91 \pm 0.06$ & $80.97 \pm 0.19$ & $80.87 \pm 0.17$ & $\mathbf{80.83 \pm 0.36}$ & $\mathbf{80.86 \pm 0.34}$ & $81.33 \pm 0.39$ & $\mathbf{82.98 \pm 1.37}$ \\ 
 & Twoing & $\mathbf{86.60 \pm 0.36}$ & $80.24 \pm 2.62$ & $75.65 \pm 1.32$ & $67.06 \pm 3.50$ & $58.73 \pm 2.62$ & $72.56 \pm 3.70$ & $64.64 \pm 1.59$ \\ 
 & Credal & $\mathbf{86.78 \pm 0.12}$ & $83.72 \pm 0.77$ & $78.69 \pm 0.95$ & $70.88 \pm 2.04$ & $59.90 \pm 2.03$ & $77.06 \pm 3.05$ & $65.44 \pm 2.60$ \\ 
 & GCE    & $86.25 \pm 0.11$ & $78.64 \pm 1.74$ & $71.75 \pm 3.22$ & $64.66 \pm 3.46$ & $57.45 \pm 1.76$ & $70.63 \pm 4.26$ & $63.44 \pm 1.63$ \\ 
 & MSE    & $86.30 \pm 0.13$ & $79.60 \pm 1.02$ & $73.04 \pm 1.32$ & $66.31 \pm 2.58$ & $57.02 \pm 4.04$ & $71.02 \pm 4.01$ & $63.96 \pm 2.79$ \\ 
 & CE     & $\mathbf{86.57 \pm 0.23}$ & $79.19 \pm 4.65$ & $73.60 \pm 1.08$ & $67.59 \pm 2.67$ & $59.64 \pm 3.93$ & $72.96 \pm 1.00$ & $64.50 \pm 3.25$ \\ 
 & Ranking & $85.73 \pm 0.30$ & $79.88 \pm 1.01$ & $74.29 \pm 2.07$ & $67.40 \pm 0.87$ & $59.23 \pm 2.08$ & $73.57 \pm 2.34$ & $64.85 \pm 1.73$ \\ 
\hline
\multirow{8}{*}{Mushrooms}
 & ANE    & $\mathbf{100.00 \pm 0.00}$ & $\mathbf{99.93 \pm 0.13}$ & $\mathbf{99.72 \pm 0.67}$ & $\mathbf{99.54 \pm 0.84}$ & $\mathbf{98.07 \pm 2.91}$ & $\mathbf{99.94 \pm 0.17}$ & $\mathbf{97.86 \pm 1.88}$ \\ 
 & CL     & $99.94 \pm 0.00$ & $\mathbf{99.95 \pm 0.10}$ & $\mathbf{99.93 \pm 0.10}$ & $\mathbf{99.63 \pm 0.53}$ & $\mathbf{98.04 \pm 1.54}$ & $\mathbf{99.77 \pm 0.18}$ & $\mathbf{98.54 \pm 1.08}$ \\ 
 & Twoing & $\mathbf{100.00 \pm 0.00}$ & $89.62 \pm 1.36$ & $80.78 \pm 1.19$ & $70.14 \pm 3.35$ & $60.04 \pm 1.27$ & $79.16 \pm 2.36$ & $69.22 \pm 2.24$ \\ 
 & Credal & $\mathbf{100.00 \pm 0.00}$ & $\mathbf{99.70 \pm 0.36}$ & $95.15 \pm 4.70$ & $83.37 \pm 2.45$ & $65.96 \pm 4.27$ & $89.60 \pm 4.33$ & $78.82 \pm 2.33$ \\ 
 & GCE    & $\mathbf{100.00 \pm 0.00}$ & $86.42 \pm 2.12$ & $76.18 \pm 3.14$ & $66.78 \pm 3.17$ & $58.35 \pm 3.88$ & $76.17 \pm 3.07$ & $67.20 \pm 1.29$ \\ 
 & MSE    & $\mathbf{100.00 \pm 0.00}$ & $84.82 \pm 1.63$ & $77.85 \pm 2.99$ & $67.34 \pm 2.44$ & $59.13 \pm 2.54$ & $75.94 \pm 1.43$ & $66.74 \pm 1.62$ \\ 
 & CE     & $\mathbf{100.00 \pm 0.00}$ & $86.54 \pm 2.68$ & $76.80 \pm 2.49$ & $66.79 \pm 3.66$ & $58.86 \pm 0.50$ & $76.32 \pm 1.37$ & $66.44 \pm 1.85$ \\ 
 & Ranking & $\mathbf{100.00 \pm 0.00}$ & $86.51 \pm 3.01$ & $75.69 \pm 2.50$ & $67.06 \pm 2.34$ & $57.53 \pm 3.64$ & $75.78 \pm 2.13$ & $65.86 \pm 2.85$ \\ 
\hline
\multirow{8}{*}{Covertype}
 & ANE    & $94.85 \pm 0.04$ & $\mathbf{89.43 \pm 0.22}$ & $\mathbf{85.85 \pm 0.39}$ & $\mathbf{79.62 \pm 0.58}$ & $\mathbf{76.38 \pm 0.37}$ & $\mathbf{81.21 \pm 0.33}$ & $\mathbf{71.65 \pm 0.52}$ \\ 
 & CL     & $77.22 \pm 0.01$ & $77.19 \pm 0.22$ & $77.13 \pm 0.30$ & $76.90 \pm 0.41$ & $\mathbf{77.01 \pm 1.19}$ & $73.70 \pm 0.57$ & $69.18 \pm 0.54$ \\ 
 & Twoing & $85.35 \pm 0.27$ & $79.23 \pm 0.40$ & $72.34 \pm 0.50$ & $65.29 \pm 0.44$ & $57.65 \pm 0.22$ & $72.23 \pm 0.23$ & $65.05 \pm 0.46$ \\ 
 & Credal & $95.26 \pm 0.06$ & $88.37 \pm 0.34$ & $78.32 \pm 0.19$ & $68.30 \pm 0.53$ & $58.87 \pm 0.26$ & $77.28 \pm 0.05$ & $67.81 \pm 0.32$ \\ 
 & GCE    & $93.86 \pm 0.04$ & $84.11 \pm 0.43$ & $75.00 \pm 0.40$ & $66.12 \pm 0.34$ & $57.84 \pm 0.26$ & $74.84 \pm 0.16$ & $66.03 \pm 0.37$ \\ 
 & MSE    & $95.07 \pm 0.08$ & $84.88 \pm 0.21$ & $75.38 \pm 0.24$ & $66.46 \pm 0.22$ & $58.01 \pm 0.13$ & $75.47 \pm 0.28$ & $66.48 \pm 0.44$ \\ 
 & CE     & $\mathbf{95.47 \pm 0.03}$ & $85.17 \pm 0.32$ & $75.71 \pm 0.21$ & $66.80 \pm 0.40$ & $58.07 \pm 0.26$ & $75.68 \pm 0.26$ & $66.61 \pm 0.33$ \\ 
 & Ranking & $95.04 \pm 0.04$ & $84.86 \pm 0.25$ & $75.68 \pm 0.38$ & $66.85 \pm 0.33$ & $58.25 \pm 0.44$ & $75.55 \pm 0.43$ & $66.52 \pm 0.37$ \\ 
\hline
\multirow{8}{*}{CIFAR-10}
 & ANE    & $\mathbf{75.29 \pm 1.01}$ & $\mathbf{73.82 \pm 0.40}$ & $\mathbf{72.48 \pm 0.57}$ & $\mathbf{70.89 \pm 0.75}$ & $\mathbf{67.72 \pm 1.11}$ & $\mathbf{67.88 \pm 1.48}$ & $\mathbf{61.97 \pm 1.01}$ \\ 
 & CL     & $\mathbf{74.90 \pm 0.14}$ & $\mathbf{74.13 \pm 0.90}$ & $\mathbf{72.72 \pm 0.73}$ & $\mathbf{71.36 \pm 0.83}$ & $\mathbf{67.97 \pm 1.51}$ & $\mathbf{68.46 \pm 1.07}$ & $\mathbf{62.22 \pm 0.97}$ \\ 
 & Twoing & $72.20 \pm 0.56$ & $67.80 \pm 1.13$ & $63.81 \pm 0.65$ & $59.44 \pm 0.88$ & $55.08 \pm 1.07$ & $66.03 \pm 0.64$ & $\mathbf{61.67 \pm 0.71}$ \\ 
 & Credal & $\mathbf{74.96 \pm 0.72}$ & $70.01 \pm 0.50$ & $65.25 \pm 0.72$ & $60.01 \pm 0.90$ & $55.12 \pm 1.94$ & $\mathbf{67.26 \pm 1.26}$ & $\mathbf{61.97 \pm 1.02}$ \\ 
 & GCE    & $73.20 \pm 0.37$ & $68.10 \pm 0.99$ & $63.59 \pm 0.86$ & $59.55 \pm 1.11$ & $55.04 \pm 0.87$ & $65.79 \pm 0.66$ & $\mathbf{61.27 \pm 0.65}$ \\ 
 & MSE    & $\mathbf{74.40 \pm 0.32}$ & $69.66 \pm 0.83$ & $64.94 \pm 0.77$ & $59.87 \pm 0.62$ & $55.26 \pm 1.13$ & $66.63 \pm 0.60$ & $\mathbf{61.85 \pm 1.04}$ \\ 
 & CE     & $\mathbf{75.02 \pm 0.36}$ & $70.09 \pm 0.88$ & $65.12 \pm 0.88$ & $59.99 \pm 0.58$ & $55.26 \pm 0.54$ & $\mathbf{67.08 \pm 0.92}$ & $\mathbf{61.96 \pm 0.85}$ \\ 
 & Ranking & $73.84 \pm 0.45$ & $69.06 \pm 1.64$ & $64.10 \pm 0.74$ & $59.79 \pm 1.19$ & $55.12 \pm 1.08$ & $66.72 \pm 0.94$ & $\mathbf{61.07 \pm 0.82}$ \\ 
\hline
\end{tabular} 
\caption{Mean test accuracy (\%) $\pm$ 2 standard deviations for DT on binary classification problems. Methods giving the largest mean accuracy and those with overlapping 97.5\% confidence interval are highlighted for each dataset and noise setting.}  
\label{tab:dt-bin}
\end{table*}

\begin{table*}[t] 
	\scriptsize 
	\centering 
	\begin{tabular}{clccccccc} 
&& \multicolumn{5}{c}{Uniform Noise} & Class Conditional \\ \cline{3-8} 
Dataset & Split Crit. & $0.0$ & $0.1$ & $0.2$ & $0.3$ & $0.4$ & Class Conditional  \\ \hline 
\multirow{7}{*}{MNIST}
 & ANE    & $86.74 \pm 0.36$ & $\mathbf{84.58 \pm 1.12}$ & $\mathbf{82.34 \pm 0.57}$ & $\mathbf{80.07 \pm 0.64}$ & $76.21 \pm 1.10$ & $74.75 \pm 1.03$ \\ 
 & CL     & $87.24 \pm 0.23$ & $\mathbf{85.33 \pm 0.58}$ & $\mathbf{83.13 \pm 0.59}$ & $\mathbf{80.53 \pm 0.67}$ & $\mathbf{77.70 \pm 0.70}$ & $\mathbf{76.65 \pm 1.17}$ \\ 
 & Twoing & $79.24 \pm 0.57$ & $72.62 \pm 0.72$ & $65.49 \pm 1.37$ & $57.83 \pm 1.28$ & $49.92 \pm 0.41$ & $50.69 \pm 0.47$ \\ 
 & Credal & $\mathbf{88.61 \pm 0.31}$ & $80.14 \pm 0.44$ & $70.41 \pm 0.96$ & $61.44 \pm 0.28$ & $52.46 \pm 1.60$ & $53.52 \pm 1.00$ \\ 
 & GCE    & $86.79 \pm 0.38$ & $75.74 \pm 0.95$ & $66.41 \pm 1.48$ & $57.22 \pm 0.47$ & $49.36 \pm 2.22$ & $50.87 \pm 0.79$ \\ 
 & MSE    & $87.89 \pm 0.32$ & $77.42 \pm 0.42$ & $67.93 \pm 0.21$ & $59.06 \pm 0.90$ & $50.81 \pm 1.49$ & $52.01 \pm 0.76$ \\ 
 & CE     & $\mathbf{88.60 \pm 0.23}$ & $77.57 \pm 1.04$ & $67.97 \pm 1.21$ & $59.32 \pm 1.05$ & $49.87 \pm 0.60$ & $51.98 \pm 0.90$ \\ 
\hline
\multirow{7}{*}{20News}
 & ANE    & $\mathbf{32.90 \pm 1.26}$ & $\mathbf{29.40 \pm 1.47}$ & $\mathbf{26.89 \pm 2.01}$ & $\mathbf{24.32 \pm 1.61}$ & $\mathbf{21.20 \pm 1.46}$ & $\mathbf{22.07 \pm 1.40}$ \\ 
 & CL     & $31.01 \pm 0.84$ & $28.64 \pm 0.48$ & $\mathbf{26.75 \pm 1.61}$ & $\mathbf{23.97 \pm 1.34}$ & $\mathbf{21.13 \pm 2.20}$ & $\mathbf{22.92 \pm 0.72}$ \\ 
 & Twoing & $25.56 \pm 0.60$ & $22.73 \pm 0.98$ & $21.95 \pm 0.73$ & $18.92 \pm 1.41$ & $17.23 \pm 0.78$ & $18.58 \pm 1.04$ \\ 
 & Credal & $33.11 \pm 0.31$ & $\mathbf{30.41 \pm 0.71}$ & $\mathbf{27.67 \pm 1.08}$ & $\mathbf{23.83 \pm 1.31}$ & $\mathbf{20.92 \pm 1.18}$ & $\mathbf{22.39 \pm 0.95}$ \\ 
 & GCE    & $31.72 \pm 0.32$ & $28.45 \pm 0.86$ & $\mathbf{25.74 \pm 1.42}$ & $\mathbf{23.11 \pm 1.10}$ & $\mathbf{20.33 \pm 0.67}$ & $\mathbf{21.85 \pm 0.73}$ \\ 
 & MSE    & $\mathbf{33.91 \pm 0.32}$ & $\mathbf{30.27 \pm 1.15}$ & $\mathbf{27.45 \pm 0.78}$ & $\mathbf{24.14 \pm 1.12}$ & $\mathbf{21.28 \pm 0.60}$ & $\mathbf{22.50 \pm 1.27}$ \\ 
 & CE     & $\mathbf{33.86 \pm 0.61}$ & $\mathbf{30.65 \pm 0.80}$ & $\mathbf{27.05 \pm 0.83}$ & $\mathbf{24.03 \pm 1.23}$ & $\mathbf{20.90 \pm 0.68}$ & $\mathbf{21.92 \pm 0.99}$ \\ 
\hline
\multirow{7}{*}{UNSW-NB15}
 & ANE    & $\mathbf{73.78 \pm 0.83}$ & $\mathbf{74.18 \pm 1.99}$ & $\mathbf{73.79 \pm 1.46}$ & $\mathbf{72.98 \pm 2.90}$ & $\mathbf{73.25 \pm 1.80}$ & $\mathbf{74.04 \pm 1.47}$ \\ 
 & CL     & $\mathbf{73.69 \pm 1.16}$ & $\mathbf{74.25 \pm 1.46}$ & $\mathbf{73.60 \pm 1.66}$ & $\mathbf{73.52 \pm 2.09}$ & $\mathbf{72.90 \pm 2.19}$ & $\mathbf{74.21 \pm 0.82}$ \\ 
 & Twoing & $70.34 \pm 4.79$ & $66.31 \pm 1.23$ & $60.79 \pm 3.60$ & $55.78 \pm 1.39$ & $49.54 \pm 2.56$ & $52.90 \pm 3.72$ \\ 
 & Credal & $\mathbf{74.49 \pm 0.11}$ & $71.67 \pm 0.86$ & $66.91 \pm 2.06$ & $61.02 \pm 0.50$ & $55.73 \pm 2.28$ & $60.43 \pm 0.59$ \\ 
 & GCE    & $71.35 \pm 0.43$ & $65.16 \pm 2.44$ & $58.46 \pm 1.98$ & $53.05 \pm 2.71$ & $46.24 \pm 1.54$ & $52.53 \pm 0.80$ \\ 
 & MSE    & $73.67 \pm 0.07$ & $66.29 \pm 1.02$ & $60.48 \pm 1.26$ & $54.79 \pm 1.62$ & $48.52 \pm 2.00$ & $54.01 \pm 1.35$ \\ 
 & CE     & $73.00 \pm 0.13$ & $66.61 \pm 1.59$ & $60.39 \pm 2.29$ & $54.55 \pm 2.34$ & $49.20 \pm 2.67$ & $54.65 \pm 0.65$ \\ 
\hline
\multirow{7}{*}{Covertype}
 & ANE    & $93.60 \pm 0.16$ & $87.03 \pm 0.52$ & $\mathbf{82.59 \pm 0.22}$ & $\mathbf{75.50 \pm 0.97}$ & $\mathbf{75.70 \pm 0.40}$ & $\mathbf{72.82 \pm 0.53}$ \\ 
 & CL     & $72.41 \pm 0.01$ & $72.50 \pm 0.32$ & $72.44 \pm 0.32$ & $72.45 \pm 0.28$ & $72.01 \pm 0.77$ & $\mathbf{72.21 \pm 0.28}$ \\ 
 & Twoing & $81.55 \pm 0.16$ & $73.92 \pm 0.35$ & $66.41 \pm 0.30$ & $59.00 \pm 0.54$ & $51.06 \pm 0.44$ & $47.33 \pm 0.40$ \\ 
 & Credal & $94.12 \pm 0.05$ & $\mathbf{89.69 \pm 0.13}$ & $\mathbf{82.62 \pm 0.51}$ & $73.96 \pm 0.29$ & $63.85 \pm 0.46$ & $57.81 \pm 0.59$ \\ 
 & GCE    & $92.77 \pm 0.06$ & $81.56 \pm 0.33$ & $71.42 \pm 0.12$ & $61.76 \pm 0.40$ & $52.54 \pm 0.22$ & $48.20 \pm 0.21$ \\ 
 & MSE    & $94.00 \pm 0.08$ & $83.07 \pm 0.20$ & $73.02 \pm 0.14$ & $63.53 \pm 0.25$ & $54.35 \pm 0.37$ & $49.37 \pm 0.63$ \\ 
 & CE     & $\mathbf{94.42 \pm 0.03}$ & $83.49 \pm 0.33$ & $73.25 \pm 0.25$ & $63.62 \pm 0.37$ & $54.19 \pm 0.64$ & $49.25 \pm 0.45$ \\ 
\hline
\multirow{7}{*}{CIFAR-10}
 & ANE    & $\mathbf{26.78 \pm 0.82}$ & $\mathbf{24.68 \pm 0.66}$ & $\mathbf{23.01 \pm 1.03}$ & $\mathbf{21.72 \pm 2.08}$ & $\mathbf{20.24 \pm 2.10}$ & $\mathbf{21.52 \pm 0.28}$ \\ 
 & CL     & $26.10 \pm 0.22$ & $\mathbf{25.62 \pm 0.99}$ & $\mathbf{24.02 \pm 0.74}$ & $\mathbf{22.59 \pm 0.95}$ & $\mathbf{21.03 \pm 0.37}$ & $\mathbf{21.99 \pm 0.57}$ \\ 
 & Twoing & $22.28 \pm 1.22$ & $21.66 \pm 0.38$ & $20.11 \pm 0.90$ & $18.69 \pm 0.54$ & $17.25 \pm 0.27$ & $18.42 \pm 0.63$ \\ 
 & Credal & $\mathbf{26.64 \pm 0.53}$ & $24.31 \pm 0.60$ & $22.99 \pm 0.55$ & $20.73 \pm 0.93$ & $19.28 \pm 0.36$ & $20.33 \pm 0.66$ \\ 
 & GCE    & $25.81 \pm 0.37$ & $23.82 \pm 0.71$ & $22.19 \pm 0.89$ & $20.48 \pm 0.78$ & $19.21 \pm 0.58$ & $20.20 \pm 0.81$ \\ 
 & MSE    & $\mathbf{26.84 \pm 0.39}$ & $\mathbf{24.95 \pm 1.19}$ & $\mathbf{22.97 \pm 0.74}$ & $20.90 \pm 0.84$ & $19.18 \pm 1.07$ & $20.77 \pm 0.58$ \\ 
 & CE     & $26.27 \pm 0.25$ & $\mathbf{24.57 \pm 0.43}$ & $22.32 \pm 1.00$ & $20.83 \pm 0.56$ & $19.00 \pm 0.81$ & $20.06 \pm 0.69$ \\ 
\hline
\end{tabular} 
\caption{Mean test accuracy (\%) $\pm$ 2 standard deviations for DT on multiclass classification problems. Methods giving the largest mean accuracy and those with overlapping 97.5\% confidence interval are highlighted for each dataset and noise setting.} 
\label{tab:dt-mc}
\end{table*}

\begin{table*}[t] 
	\scriptsize 
	\centering 
	\begin{tabular}{clccccccc} 
&& \multicolumn{5}{c}{Uniform Noise} & \multicolumn{2}{c}{Class Conditional} \\ \cline{3-9} 
Dataset & Split Crit. & $0.0$ & $0.1$ & $0.2$ & $0.3$ & $0.4$ & $(0.1,0.3)$ & $(0.2,0.4)$  \\ \hline 
\multirow{8}{*}{MNIST}
 & ANE    & $98.04 \pm 0.12$ & $97.89 \pm 0.10$ & $97.54 \pm 0.15$ & $95.63 \pm 0.43$ & $\mathbf{88.67 \pm 1.58}$ & $\mathbf{95.91 \pm 0.56}$ & $\mathbf{88.96 \pm 1.06}$ \\ 
 & CL     & $91.83 \pm 0.85$ & $91.82 \pm 0.76$ & $91.76 \pm 0.64$ & $91.53 \pm 0.52$ & $\mathbf{89.13 \pm 1.29}$ & $85.40 \pm 0.75$ & $78.53 \pm 2.12$ \\ 
 & Twoing & $97.62 \pm 0.20$ & $97.61 \pm 0.21$ & $97.46 \pm 0.14$ & $\mathbf{96.42 \pm 0.37}$ & $\mathbf{87.72 \pm 1.23}$ & $\mathbf{95.64 \pm 0.33}$ & $\mathbf{88.46 \pm 0.83}$ \\ 
 & Credal & $98.10 \pm 0.06$ & $\mathbf{98.03 \pm 0.13}$ & $\mathbf{97.70 \pm 0.17}$ & $95.74 \pm 0.45$ & $84.72 \pm 0.82$ & $\mathbf{96.20 \pm 0.69}$ & $\mathbf{89.74 \pm 0.98}$ \\ 
 & GCE    & $\mathbf{98.06 \pm 0.14}$ & $\mathbf{98.00 \pm 0.10}$ & $97.03 \pm 0.29$ & $93.85 \pm 0.29$ & $82.17 \pm 0.62$ & $94.75 \pm 0.29$ & $87.65 \pm 1.36$ \\ 
 & MSE    & $\mathbf{98.16 \pm 0.11}$ & $\mathbf{98.14 \pm 0.17}$ & $\mathbf{97.73 \pm 0.16}$ & $\mathbf{95.79 \pm 0.54}$ & $84.76 \pm 1.44$ & $\mathbf{95.98 \pm 0.38}$ & $\mathbf{88.69 \pm 1.55}$ \\ 
 & CE     & $\mathbf{98.18 \pm 0.03}$ & $\mathbf{98.10 \pm 0.07}$ & $\mathbf{97.83 \pm 0.13}$ & $\mathbf{96.15 \pm 0.22}$ & $85.69 \pm 0.99$ & $\mathbf{96.14 \pm 0.60}$ & $\mathbf{88.24 \pm 1.31}$ \\ 
 & Ranking & $98.05 \pm 0.14$ & $\mathbf{97.89 \pm 0.17}$ & $\mathbf{97.61 \pm 0.17}$ & $\mathbf{95.86 \pm 0.50}$ & $85.08 \pm 1.50$ & $\mathbf{95.81 \pm 0.62}$ & $\mathbf{88.37 \pm 0.87}$ \\ 
\hline
\multirow{8}{*}{20News}
 & ANE    & $\mathbf{86.05 \pm 0.63}$ & $\mathbf{85.88 \pm 0.34}$ & $\mathbf{85.09 \pm 0.60}$ & $\mathbf{82.05 \pm 1.76}$ & $\mathbf{73.84 \pm 3.15}$ & $\mathbf{80.68 \pm 0.75}$ & $\mathbf{70.82 \pm 2.80}$ \\ 
 & CL     & $81.70 \pm 0.32$ & $81.55 \pm 0.31$ & $80.18 \pm 1.04$ & $77.63 \pm 1.03$ & $71.01 \pm 2.22$ & $76.55 \pm 1.95$ & $65.60 \pm 3.51$ \\ 
 & Twoing & $84.88 \pm 0.05$ & $84.54 \pm 0.54$ & $83.99 \pm 0.25$ & $\mathbf{82.16 \pm 0.49}$ & $\mathbf{73.52 \pm 1.15}$ & $\mathbf{79.41 \pm 1.30}$ & $\mathbf{71.49 \pm 1.34}$ \\ 
 & Credal & $\mathbf{85.96 \pm 0.22}$ & $\mathbf{85.72 \pm 0.39}$ & $\mathbf{84.72 \pm 0.46}$ & $\mathbf{82.53 \pm 0.67}$ & $\mathbf{73.72 \pm 2.41}$ & $\mathbf{80.86 \pm 1.49}$ & $\mathbf{73.47 \pm 1.50}$ \\ 
 & GCE    & $85.04 \pm 0.25$ & $83.72 \pm 0.52$ & $81.16 \pm 1.07$ & $77.11 \pm 0.70$ & $69.44 \pm 1.59$ & $78.05 \pm 1.60$ & $70.79 \pm 1.60$ \\ 
 & MSE    & $\mathbf{86.06 \pm 0.39}$ & $\mathbf{85.78 \pm 0.30}$ & $\mathbf{84.60 \pm 0.48}$ & $\mathbf{82.29 \pm 1.27}$ & $\mathbf{73.09 \pm 1.16}$ & $\mathbf{80.57 \pm 1.50}$ & $\mathbf{72.57 \pm 2.85}$ \\ 
 & CE     & $\mathbf{86.25 \pm 0.63}$ & $\mathbf{85.88 \pm 0.42}$ & $\mathbf{85.07 \pm 0.93}$ & $\mathbf{83.10 \pm 1.01}$ & $\mathbf{74.86 \pm 2.28}$ & $\mathbf{80.89 \pm 1.16}$ & $\mathbf{72.06 \pm 3.21}$ \\ 
 & Ranking & $\mathbf{86.15 \pm 0.16}$ & $\mathbf{85.93 \pm 0.19}$ & $\mathbf{85.10 \pm 0.49}$ & $\mathbf{82.20 \pm 0.62}$ & $\mathbf{74.15 \pm 3.47}$ & $\mathbf{80.01 \pm 1.57}$ & $70.72 \pm 0.96$ \\ 
\hline
\multirow{8}{*}{UNSW-NB15}
 & ANE    & $\mathbf{87.32 \pm 0.09}$ & $\mathbf{86.95 \pm 0.25}$ & $82.57 \pm 0.90$ & $\mathbf{86.44 \pm 0.45}$ & $\mathbf{80.02 \pm 1.87}$ & $\mathbf{89.64 \pm 0.51}$ & $\mathbf{89.29 \pm 1.96}$ \\ 
 & CL     & $80.31 \pm 1.10$ & $80.21 \pm 1.06$ & $80.59 \pm 1.21$ & $80.69 \pm 0.80$ & $\mathbf{79.81 \pm 1.58}$ & $85.14 \pm 3.25$ & $\mathbf{89.79 \pm 1.05}$ \\ 
 & Twoing & $\mathbf{87.33 \pm 0.07}$ & $\mathbf{86.65 \pm 0.22}$ & $85.64 \pm 0.46$ & $82.83 \pm 0.92$ & $72.60 \pm 1.63$ & $86.46 \pm 0.78$ & $79.65 \pm 2.24$ \\ 
 & Credal & $86.98 \pm 0.10$ & $\mathbf{87.01 \pm 0.25}$ & $\mathbf{86.47 \pm 0.58}$ & $83.46 \pm 0.72$ & $72.68 \pm 1.92$ & $87.90 \pm 0.67$ & $81.21 \pm 2.08$ \\ 
 & GCE    & $86.77 \pm 0.10$ & $85.92 \pm 0.26$ & $83.90 \pm 0.57$ & $79.16 \pm 1.71$ & $68.86 \pm 2.48$ & $83.91 \pm 1.09$ & $77.43 \pm 3.33$ \\ 
 & MSE    & $87.09 \pm 0.12$ & $86.36 \pm 0.37$ & $84.93 \pm 0.37$ & $80.97 \pm 1.01$ & $70.58 \pm 2.59$ & $85.57 \pm 1.46$ & $78.00 \pm 3.15$ \\ 
 & CE     & $\mathbf{87.29 \pm 0.05}$ & $86.54 \pm 0.24$ & $85.21 \pm 0.55$ & $81.42 \pm 1.20$ & $71.12 \pm 3.50$ & $85.97 \pm 1.29$ & $79.38 \pm 1.99$ \\ 
 & Ranking & $87.10 \pm 0.09$ & $86.27 \pm 0.22$ & $85.06 \pm 0.77$ & $81.71 \pm 0.95$ & $71.32 \pm 1.72$ & $85.95 \pm 0.31$ & $80.50 \pm 1.82$ \\ 
\hline
\multirow{8}{*}{Mushrooms}
 & ANE    & $\mathbf{100.00 \pm 0.00}$ & $\mathbf{99.79 \pm 0.24}$ & $\mathbf{99.54 \pm 0.73}$ & $\mathbf{99.29 \pm 0.76}$ & $\mathbf{98.18 \pm 1.15}$ & $\mathbf{99.16 \pm 1.78}$ & $\mathbf{93.70 \pm 6.20}$ \\ 
 & CL     & $99.27 \pm 0.22$ & $\mathbf{99.46 \pm 0.59}$ & $\mathbf{99.31 \pm 0.25}$ & $\mathbf{99.08 \pm 0.42}$ & $\mathbf{97.77 \pm 0.76}$ & $\mathbf{99.77 \pm 0.78}$ & $\mathbf{95.32 \pm 3.89}$ \\ 
 & Twoing & $\mathbf{100.00 \pm 0.00}$ & $98.65 \pm 0.87$ & $93.33 \pm 2.33$ & $83.50 \pm 2.21$ & $69.91 \pm 3.24$ & $90.95 \pm 1.51$ & $79.22 \pm 3.10$ \\ 
 & Credal & $\mathbf{100.00 \pm 0.00}$ & $\mathbf{99.93 \pm 0.27}$ & $\mathbf{99.36 \pm 0.57}$ & $93.14 \pm 2.10$ & $74.20 \pm 3.63$ & $96.71 \pm 0.92$ & $86.17 \pm 3.06$ \\ 
 & GCE    & $\mathbf{100.00 \pm 0.00}$ & $93.53 \pm 1.31$ & $84.69 \pm 1.83$ & $74.20 \pm 3.19$ & $62.57 \pm 2.90$ & $83.58 \pm 2.14$ & $73.75 \pm 2.04$ \\ 
 & MSE    & $\mathbf{100.00 \pm 0.00}$ & $94.09 \pm 0.73$ & $85.98 \pm 1.59$ & $75.88 \pm 1.76$ & $62.86 \pm 0.89$ & $84.89 \pm 1.06$ & $73.90 \pm 2.27$ \\ 
 & CE     & $\mathbf{100.00 \pm 0.00}$ & $94.25 \pm 0.77$ & $85.35 \pm 2.48$ & $76.14 \pm 2.97$ & $63.73 \pm 1.86$ & $84.47 \pm 1.81$ & $74.98 \pm 1.89$ \\ 
 & Ranking & $\mathbf{100.00 \pm 0.00}$ & $94.41 \pm 1.52$ & $85.69 \pm 1.16$ & $75.14 \pm 0.80$ & $63.67 \pm 1.91$ & $84.05 \pm 2.57$ & $73.70 \pm 2.14$ \\ 
\hline
\multirow{8}{*}{Covertype}
 & ANE    & $97.69 \pm 0.04$ & $95.47 \pm 0.10$ & $91.23 \pm 0.13$ & $\mathbf{89.33 \pm 0.15}$ & $\mathbf{78.79 \pm 0.47}$ & $86.74 \pm 0.20$ & $\mathbf{79.02 \pm 2.13}$ \\ 
 & CL     & $79.03 \pm 0.99$ & $78.96 \pm 1.16$ & $79.17 \pm 1.24$ & $79.54 \pm 1.03$ & $\mathbf{78.84 \pm 0.82}$ & $74.39 \pm 0.84$ & $63.53 \pm 1.56$ \\ 
 & Twoing & $93.24 \pm 0.07$ & $92.61 \pm 0.22$ & $90.40 \pm 0.17$ & $84.57 \pm 0.16$ & $71.59 \pm 0.17$ & $86.30 \pm 0.19$ & $\mathbf{78.25 \pm 0.29}$ \\ 
 & Credal & $97.75 \pm 0.03$ & $\mathbf{96.38 \pm 0.08}$ & $\mathbf{91.51 \pm 0.21}$ & $81.69 \pm 0.39$ & $67.51 \pm 0.38$ & $\mathbf{88.33 \pm 0.15}$ & $\mathbf{78.72 \pm 0.27}$ \\ 
 & GCE    & $97.58 \pm 0.04$ & $94.94 \pm 0.23$ & $89.07 \pm 0.32$ & $79.38 \pm 0.19$ & $66.21 \pm 0.17$ & $86.12 \pm 0.29$ & $76.67 \pm 0.25$ \\ 
 & MSE    & $97.79 \pm 0.02$ & $95.05 \pm 0.12$ & $89.20 \pm 0.25$ & $79.56 \pm 0.17$ & $66.18 \pm 0.38$ & $86.39 \pm 0.17$ & $76.79 \pm 0.27$ \\ 
 & CE     & $\mathbf{97.87 \pm 0.04}$ & $95.27 \pm 0.11$ & $89.47 \pm 0.17$ & $80.13 \pm 0.24$ & $66.77 \pm 0.37$ & $86.71 \pm 0.35$ & $\mathbf{77.24 \pm 0.22}$ \\ 
 & Ranking & $\mathbf{97.90 \pm 0.04}$ & $95.34 \pm 0.14$ & $89.61 \pm 0.21$ & $80.15 \pm 0.35$ & $66.55 \pm 0.15$ & $86.76 \pm 0.07$ & $\mathbf{77.01 \pm 0.44}$ \\ 
\hline
\multirow{8}{*}{CIFAR-10}
 & ANE    & $85.28 \pm 0.28$ & $\mathbf{85.23 \pm 0.26}$ & $\mathbf{84.68 \pm 0.46}$ & $\mathbf{83.05 \pm 0.68}$ & $\mathbf{76.07 \pm 1.54}$ & $\mathbf{80.11 \pm 0.79}$ & $\mathbf{75.12 \pm 1.08}$ \\ 
 & CL     & $79.04 \pm 0.23$ & $79.10 \pm 0.26$ & $78.78 \pm 0.92$ & $77.30 \pm 0.29$ & $72.07 \pm 0.84$ & $67.86 \pm 1.14$ & $60.60 \pm 0.46$ \\ 
 & Twoing & $83.24 \pm 0.20$ & $82.91 \pm 0.19$ & $82.82 \pm 0.42$ & $81.81 \pm 0.67$ & $\mathbf{75.87 \pm 0.80}$ & $77.32 \pm 0.40$ & $72.79 \pm 1.15$ \\ 
 & Credal & $\mathbf{85.62 \pm 0.10}$ & $\mathbf{85.49 \pm 0.33}$ & $\mathbf{84.82 \pm 0.51}$ & $\mathbf{82.87 \pm 0.73}$ & $\mathbf{75.31 \pm 1.03}$ & $\mathbf{80.77 \pm 0.38}$ & $\mathbf{76.54 \pm 0.88}$ \\ 
 & GCE    & $85.06 \pm 0.17$ & $83.79 \pm 0.62$ & $81.10 \pm 0.67$ & $77.00 \pm 0.84$ & $70.62 \pm 1.18$ & $78.49 \pm 0.48$ & $73.99 \pm 1.33$ \\ 
 & MSE    & $\mathbf{85.69 \pm 0.21}$ & $\mathbf{85.40 \pm 0.43}$ & $\mathbf{84.49 \pm 0.35}$ & $82.61 \pm 0.59$ & $\mathbf{75.10 \pm 0.44}$ & $\mathbf{80.59 \pm 0.71}$ & $\mathbf{75.82 \pm 0.72}$ \\ 
 & CE     & $\mathbf{85.57 \pm 0.28}$ & $\mathbf{85.50 \pm 0.44}$ & $\mathbf{84.97 \pm 0.70}$ & $\mathbf{83.34 \pm 0.31}$ & $\mathbf{76.39 \pm 1.26}$ & $\mathbf{80.27 \pm 0.70}$ & $\mathbf{75.84 \pm 0.50}$ \\ 
 & Ranking & $84.77 \pm 0.16$ & $84.66 \pm 0.28$ & $83.98 \pm 0.55$ & $\mathbf{82.49 \pm 0.83}$ & $\mathbf{75.63 \pm 0.67}$ & $79.40 \pm 0.59$ & $74.84 \pm 0.90$ \\ 
\hline
\end{tabular} 
\caption{Mean test accuracy (\%) $\pm$ 2 standard deviations for RF on binary classification problems. Methods giving the largest mean accuracy and those with overlapping 97.5\% confidence interval are highlighted for each dataset and noise setting.} 
\label{tab:rf-bin}
\end{table*}

\begin{table*}[t] 
	\scriptsize 
	\centering 
	\begin{tabular}{clccccccc} 
&& \multicolumn{5}{c}{Uniform Noise} & Class Conditional \\ \cline{3-8} 
Dataset & Split Crit. & $0.0$ & $0.1$ & $0.2$ & $0.3$ & $0.4$ & Class Conditional  \\ \hline 
\multirow{7}{*}{MNIST}
 & ANE    & $97.06 \pm 0.24$ & $96.95 \pm 0.20$ & $96.82 \pm 0.17$ & $96.70 \pm 0.28$ & $96.38 \pm 0.17$ & $95.93 \pm 0.28$ \\ 
 & CL     & $96.35 \pm 0.15$ & $96.41 \pm 0.21$ & $96.39 \pm 0.10$ & $96.24 \pm 0.11$ & $96.06 \pm 0.10$ & $95.44 \pm 0.50$ \\ 
 & Twoing & $96.42 \pm 0.20$ & $96.41 \pm 0.19$ & $96.27 \pm 0.15$ & $96.18 \pm 0.21$ & $95.85 \pm 0.06$ & $95.41 \pm 0.23$ \\ 
 & Credal & $\mathbf{97.21 \pm 0.13}$ & $\mathbf{97.15 \pm 0.11}$ & $\mathbf{97.12 \pm 0.36}$ & $96.94 \pm 0.05$ & $\mathbf{96.70 \pm 0.24}$ & $\mathbf{96.30 \pm 0.20}$ \\ 
 & GCE    & $\mathbf{97.29 \pm 0.18}$ & $\mathbf{97.23 \pm 0.13}$ & $\mathbf{97.13 \pm 0.19}$ & $96.80 \pm 0.16$ & $95.94 \pm 0.41$ & $95.39 \pm 0.19$ \\ 
 & MSE    & $\mathbf{97.40 \pm 0.16}$ & $\mathbf{97.36 \pm 0.16}$ & $\mathbf{97.18 \pm 0.22}$ & $\mathbf{97.07 \pm 0.12}$ & $\mathbf{96.66 \pm 0.16}$ & $\mathbf{96.20 \pm 0.32}$ \\ 
 & CE     & $\mathbf{97.37 \pm 0.17}$ & $\mathbf{97.27 \pm 0.09}$ & $\mathbf{97.13 \pm 0.19}$ & $96.87 \pm 0.10$ & $\mathbf{96.51 \pm 0.24}$ & $\mathbf{96.06 \pm 0.15}$ \\ 
\hline
\multirow{7}{*}{20News}
 & ANE    & $\mathbf{63.15 \pm 0.51}$ & $\mathbf{62.33 \pm 0.62}$ & $\mathbf{61.42 \pm 1.07}$ & $\mathbf{59.93 \pm 0.89}$ & $\mathbf{58.12 \pm 0.51}$ & $\mathbf{56.49 \pm 1.05}$ \\ 
 & CL     & $\mathbf{62.79 \pm 0.56}$ & $61.70 \pm 0.48$ & $60.62 \pm 1.01$ & $58.32 \pm 0.66$ & $55.98 \pm 0.78$ & $54.53 \pm 1.30$ \\ 
 & Twoing & $58.01 \pm 0.65$ & $57.13 \pm 0.71$ & $55.81 \pm 0.57$ & $54.32 \pm 0.76$ & $52.26 \pm 0.67$ & $51.27 \pm 0.96$ \\ 
 & Credal & $\mathbf{63.28 \pm 0.49}$ & $\mathbf{62.95 \pm 1.12}$ & $\mathbf{62.00 \pm 0.51}$ & $\mathbf{60.80 \pm 0.31}$ & $\mathbf{58.65 \pm 0.62}$ & $\mathbf{57.59 \pm 0.83}$ \\ 
 & GCE    & $\mathbf{63.17 \pm 0.71}$ & $61.82 \pm 0.64$ & $60.59 \pm 0.20$ & $58.48 \pm 0.96$ & $55.97 \pm 1.05$ & $54.95 \pm 0.85$ \\ 
 & MSE    & $\mathbf{63.69 \pm 0.73}$ & $\mathbf{63.02 \pm 0.59}$ & $\mathbf{61.70 \pm 0.39}$ & $\mathbf{60.69 \pm 0.56}$ & $\mathbf{58.23 \pm 0.88}$ & $\mathbf{57.25 \pm 0.70}$ \\ 
 & CE     & $\mathbf{63.36 \pm 0.61}$ & $\mathbf{62.36 \pm 0.99}$ & $61.04 \pm 0.67$ & $59.69 \pm 0.73$ & $57.58 \pm 0.68$ & $55.98 \pm 0.60$ \\ 
\hline
\multirow{7}{*}{UNSW-NB15}
 & ANE    & $75.50 \pm 0.13$ & $\mathbf{75.70 \pm 0.13}$ & $75.09 \pm 0.32$ & $\mathbf{75.08 \pm 0.45}$ & $\mathbf{74.88 \pm 0.74}$ & $\mathbf{75.49 \pm 0.57}$ \\ 
 & CL     & $75.58 \pm 0.54$ & $\mathbf{75.45 \pm 0.51}$ & $\mathbf{75.44 \pm 0.41}$ & $\mathbf{75.52 \pm 0.15}$ & $\mathbf{75.00 \pm 0.99}$ & $\mathbf{75.71 \pm 0.56}$ \\ 
 & Twoing & $74.70 \pm 0.14$ & $74.38 \pm 0.17$ & $73.85 \pm 0.22$ & $72.88 \pm 0.39$ & $71.32 \pm 0.27$ & $71.82 \pm 0.29$ \\ 
 & Credal & $\mathbf{76.85 \pm 0.11}$ & $\mathbf{75.90 \pm 0.20}$ & $\mathbf{75.55 \pm 0.26}$ & $\mathbf{75.02 \pm 0.72}$ & $\mathbf{74.33 \pm 0.71}$ & $\mathbf{74.86 \pm 0.70}$ \\ 
 & GCE    & $75.36 \pm 0.09$ & $74.63 \pm 0.22$ & $73.73 \pm 0.31$ & $72.52 \pm 0.43$ & $70.08 \pm 0.48$ & $71.52 \pm 0.70$ \\ 
 & MSE    & $75.53 \pm 0.12$ & $74.94 \pm 0.07$ & $74.14 \pm 0.46$ & $73.40 \pm 0.47$ & $71.59 \pm 0.50$ & $72.59 \pm 0.24$ \\ 
 & CE     & $75.75 \pm 0.21$ & $75.28 \pm 0.20$ & $74.61 \pm 0.32$ & $73.75 \pm 0.23$ & $72.11 \pm 0.46$ & $72.93 \pm 0.36$ \\ 
\hline
\multirow{7}{*}{Covertype}
 & ANE    & $97.18 \pm 0.06$ & $95.54 \pm 0.10$ & $95.31 \pm 0.18$ & $\mathbf{93.60 \pm 0.14}$ & $\mathbf{89.23 \pm 0.09}$ & $\mathbf{85.00 \pm 0.38}$ \\ 
 & CL     & $73.94 \pm 0.62$ & $73.80 \pm 0.32$ & $73.80 \pm 0.51$ & $73.75 \pm 0.04$ & $73.83 \pm 0.22$ & $74.13 \pm 0.73$ \\ 
 & Twoing & $90.72 \pm 0.10$ & $90.30 \pm 0.22$ & $89.53 \pm 0.20$ & $87.95 \pm 0.16$ & $84.87 \pm 0.12$ & $78.82 \pm 0.19$ \\ 
 & Credal & $97.07 \pm 0.05$ & $\mathbf{96.78 \pm 0.05}$ & $\mathbf{95.59 \pm 0.07}$ & $92.84 \pm 0.13$ & $87.29 \pm 0.07$ & $80.00 \pm 0.30$ \\ 
 & GCE    & $97.01 \pm 0.04$ & $95.11 \pm 0.08$ & $94.87 \pm 0.04$ & $92.52 \pm 0.06$ & $87.81 \pm 0.24$ & $79.60 \pm 0.23$ \\ 
 & MSE    & $97.29 \pm 0.03$ & $95.40 \pm 0.15$ & $95.37 \pm 0.05$ & $92.88 \pm 0.16$ & $88.12 \pm 0.07$ & $79.87 \pm 0.25$ \\ 
 & CE     & $\mathbf{97.47 \pm 0.02}$ & $95.67 \pm 0.09$ & $\mathbf{95.58 \pm 0.12}$ & $93.19 \pm 0.09$ & $88.41 \pm 0.19$ & $80.19 \pm 0.11$ \\ 
\hline
\multirow{7}{*}{CIFAR-10}
 & ANE    & $47.53 \pm 0.21$ & $\mathbf{46.90 \pm 0.35}$ & $\mathbf{46.17 \pm 0.61}$ & $\mathbf{45.25 \pm 0.34}$ & $\mathbf{43.84 \pm 1.07}$ & $\mathbf{42.76 \pm 0.62}$ \\ 
 & CL     & $44.93 \pm 0.34$ & $43.94 \pm 0.61$ & $43.22 \pm 0.32$ & $42.02 \pm 0.40$ & $39.74 \pm 0.69$ & $38.87 \pm 0.45$ \\ 
 & Twoing & $42.94 \pm 0.73$ & $42.63 \pm 0.80$ & $41.89 \pm 0.71$ & $40.86 \pm 0.25$ & $39.56 \pm 0.51$ & $38.48 \pm 0.80$ \\ 
 & Credal & $\mathbf{48.00 \pm 0.65}$ & $\mathbf{47.46 \pm 0.69}$ & $\mathbf{46.44 \pm 0.38}$ & $\mathbf{45.52 \pm 0.79}$ & $\mathbf{43.76 \pm 0.46}$ & $\mathbf{42.85 \pm 0.57}$ \\ 
 & GCE    & $46.85 \pm 0.31$ & $45.94 \pm 0.45$ & $44.89 \pm 0.47$ & $43.03 \pm 0.51$ & $41.55 \pm 0.75$ & $40.55 \pm 0.71$ \\ 
 & MSE    & $\mathbf{48.31 \pm 0.43}$ & $\mathbf{47.56 \pm 0.65}$ & $\mathbf{47.16 \pm 0.97}$ & $\mathbf{45.93 \pm 0.74}$ & $\mathbf{44.07 \pm 0.61}$ & $\mathbf{43.21 \pm 1.21}$ \\ 
 & CE     & $\mathbf{47.69 \pm 0.49}$ & $\mathbf{47.04 \pm 0.56}$ & $\mathbf{46.31 \pm 0.50}$ & $\mathbf{45.19 \pm 0.82}$ & $\mathbf{43.32 \pm 0.44}$ & $\mathbf{41.88 \pm 0.90}$ \\ 
\hline
\end{tabular} 
\caption{Mean test accuracy (\%) $\pm$ 2 standard deviations for RF on multiclass classification problems. Methods giving the largest mean accuracy and those with overlapping 97.5\% confidence interval are highlighted for each dataset and noise setting.} 
\label{tab:rf-mc}
\end{table*}

\section{Early Stopping in Binary Classification}
\label{app:alternate_early_stopping}

\Cref{fig:threshold_tuning} shows that generally a larger $\lambda$ leads to smaller trees.
This is expected, because a larger $\lambda$ is associated with higher similarity with the 
misclassification impurity, thus the early stopping behavior in \Cref{thm:early_stopping} 
is more likely to occur, resulting in a smaller tree.
However, we also observed that in some cases, a small $\lambda$ leads to the smallest trees.
This is because a small $\lambda$ has an alternative mechanism of encouraging a smaller tree:
it can favor the production of pure child nodes, 
thus reducing the need to split in future.
On the other hand, this behavior is less likely to happen for a larger $\lambda$.

We explain this using \Cref{fig:alternate_early_stopping} to compare how two different splits 
${\color{blue}(L_{1}, R_{1})}$ and ${\color{red}(L_{2}, R_{2})}$
are ranked for $\lambda=0$ and $\lambda=1$, in a binary classification setting.
The split ${\color{blue}(L_{1}, R_{1})}$  creates a pure left child node $L_{1}$ with the proportion of positive
examples in $L_{1}$ being $p_{L_{1}} = 0$, while $p_{R_{1}} = 0.67$.
The split ${\color{red}(L_{2}, R_{2})}$ does not create any pure child node, with
$p_{L_{2}} = 0.2$ and $p_{R_{2}} = 0.8$.
The left subfigure shows that when $\lambda=0$, the split with higher risk reduction is ${\color{blue}(L_{1}, R_{1})}$,
which creates a pure node.
The right subfigure shows that when $\lambda=1$, the split with higher risk reduction is ${\color{red}(L_{2}, R_{2})}$,
which does not create a pure node.

\begin{figure}[h]
    \centering
    \includegraphics[height=3.6cm]{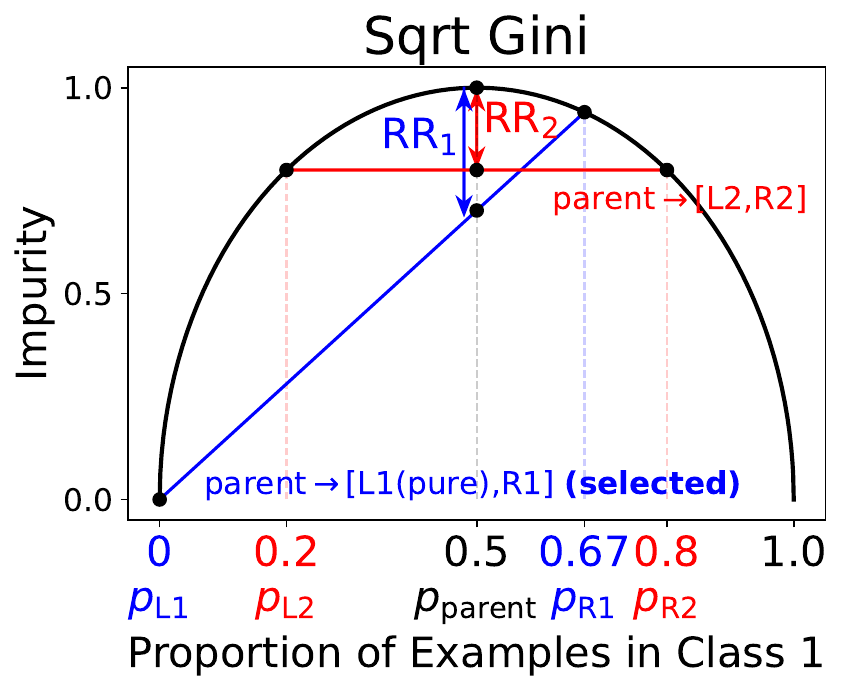}
    \includegraphics[height=3.6cm]{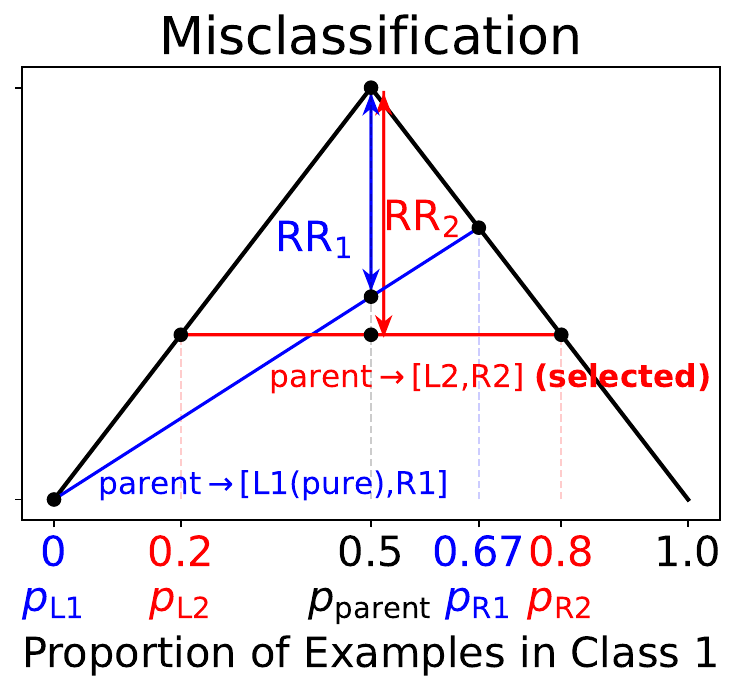}
    \caption{
    Left: NE loss with $\lambda=0$ ranks the split producing a pure child node above the one that does not. 
    Right: NE loss with $\lambda=1$ gives opposite rankings.
    }
    \label{fig:alternate_early_stopping}
\end{figure}

Note that the black curve plots the impurity for a node with a class distribution $(p, 1-p)$ against $p$.
We use the following trick to read the risk reduction from the plot:
consider a node with a class distribution $(p_{\text{parent}}, 1-p_{\text{parent}})$ and 
a split which creates two child nodes with class distributions ${\color{blue}(p_{L}, 1-p_{L})}$ and 
${\color{red}(p_{R}, 1-p_{R})}$.
Let $A$ and $B$ be the two points of the impurity curve at $p_{L}$ and $p_{R}$, then
the risk reduction is the gap between the impurity curve and the line segment $AB$, at $p=p_{\text{parent}}$.

\end{document}